\documentclass[sigconf, nonacm]{acmart}
\newcommand\vldbdoi{XX.XX/XXX.XX}
\newcommand\vldbpages{XXX-XXX}
\newcommand\vldbvolume{14}
\newcommand\vldbissue{1}
\newcommand\vldbyear{2020}
\newcommand\vldbauthors{\authors}
\newcommand\vldbtitle{\shorttitle} 
\newcommand\vldbavailabilityurl{http://vldb.org/pvldb/format_vol14.html}
\newcommand\vldbpagestyle{plain} 

\newtheorem{theorem}{Theorem}
\newtheorem{example}{Example}
\newtheorem{definition}{Definition}
\newtheorem{proposition}{Proposition}

\newtheorem{lemma}{Lemma}

\usepackage{cases}
\usepackage[overload]{empheq}
\usepackage{subfig}
\usepackage{graphicx}

\newcommand{\figwidthone}{1,0} 
\usepackage{amsmath}
\usepackage{mathrsfs}
\usepackage[lined,boxed,commentsnumbered,ruled, linesnumbered]{algorithm2e}
\usepackage{algorithmic}
\usepackage{color}
\usepackage{makecell}
\usepackage{bm}
\usepackage{url}
\usepackage{hyperref}
\usepackage{mathtools}
\usepackage{balance}
\usepackage{nopageno}

\newcommand{\mc}[1]{\textcolor{black}{#1}} 

\usepackage{xcolor}
\usepackage{color, soul}

\newcommand{\nop}[1]{}

\begin{document}
\begin{sloppy}
\title{Comprehensible Counterfactual Explanation on Kolmogorov-Smirnov Test}

\author{Zicun Cong}
\affiliation{%
  \institution{Simon Fraser University}
}
\email{zcong@sfu.ca}

\author{Lingyang Chu}
\affiliation{%
  \institution{McMaster University}
}
\email{chul9@mcmaster.ca}

\author{Yu Yang}
\affiliation{%
  \institution{City University of Hong Kong}
}
\email{yuyang@cityu.edu.hk}

\author{Jian Pei}
\affiliation{%
  \institution{Simon Fraser University}
}
\email{jpei@cs.sfu.ca}

\begin{abstract}

The Kolmogorov-Smirnov (KS) test is popularly used in many applications, such as anomaly detection, astronomy, database security and AI systems.  One challenge remained untouched is how we can obtain an explanation on why a test set fails the KS test. In this paper, we tackle the problem of producing counterfactual explanations for test data failing the KS test. Concept-wise, we propose the notion of most comprehensible counterfactual explanations, which accommodates both the KS test data and the user domain knowledge in producing explanations.  Computation-wise, we develop an efficient algorithm MOCHE \mc{(for \underline{MO}st \underline{C}ompre\underline{H}ensible \underline{E}xplanation)} that avoids enumerating and checking an exponential number of subsets of the test set failing the KS test.  MOCHE not only guarantees to produce the most comprehensible counterfactual explanations, but also is orders of magnitudes faster than the baselines. Experiment-wise, we present a systematic empirical study on a series of benchmark real datasets to verify the effectiveness, efficiency and scalability of most comprehensible counterfactual explanations and MOCHE.

\end{abstract}

\maketitle

\pagestyle{\vldbpagestyle}
\begingroup\small\noindent\raggedright\textbf{PVLDB Reference Format:}\\
\vldbauthors. \vldbtitle. PVLDB, \vldbvolume(\vldbissue): \vldbpages, \vldbyear.\\
\href{https://doi.org/\vldbdoi}{doi:\vldbdoi}
\endgroup
\begingroup
\renewcommand\thefootnote{}\footnote{\noindent
This work is licensed under the Creative Commons BY-NC-ND 4.0 International License. Visit \url{https://creativecommons.org/licenses/by-nc-nd/4.0/} to view a copy of this license. For any use beyond those covered by this license, obtain permission by emailing \href{mailto:info@vldb.org}{info@vldb.org}. Copyright is held by the owner/author(s). Publication rights licensed to the VLDB Endowment. \\
\raggedright Proceedings of the VLDB Endowment, Vol. \vldbvolume, No. \vldbissue\ %
ISSN 2150-8097. \\
\href{https://doi.org/\vldbdoi}{doi:\vldbdoi} \\
}\addtocounter{footnote}{-1}\endgroup

\ifdefempty{\vldbavailabilityurl}{}{
\vspace{.3cm}
\begingroup\small\noindent\raggedright\textbf{PVLDB Artifact Availability:}\\
The source code, data, and/or other artifacts have been made available at \url{\vldbavailabilityurl}.
\endgroup
}

\section{Introduction}
\label{sec:intro}

The well-known Kolmogorov-Smirnov (KS) test~\cite{klotz1967asymptotic} is a statistical hypothesis test that checks whether a test set is sampled from the same probability distribution as a reference set. If a reference set and a test set fail the KS test, it indicates that the two sets are unlikely from the same probability distribution. \mc{The KS test has been widely used to detect differences, changes and abnormalies in many areas, such as astronomy~\cite{naess2012application}, database security~\cite{santos2014approaches} and AI systems~\cite{rabanser2019failing}. Many important decisions are made based on the raised alarms about changes and abnormality, such as updating an AI model~\cite{yu2018request} and overhauling a manufacture line~\cite{kifer2004detecting}. Understanding  why a test set fails a KS test can build trust from users~\cite{ribeiro2016should} and thus improve the decision quality. In some situations, the understanding may help save data labeling and model construction costs~\cite{yu2018request}.}
However, a failed KS test itself does not come with an explanation on which data points in the test set cause the failure.

\mc{Counterfactual explanations~\cite{moraffah2020causal} have been widely adopted to interpret algorithmic decisions in many real world applications~\cite{brundage2020toward, fong2017interpretable, mothilal2020explaining, wachter2017counterfactual}, due to its beauty of being concise and easy to understand~\cite{sokol2019counterfactual, moraffah2020causal}.
A counterfactual explanation of a decision $Y$ is the smallest set of relevant factors $X$ such that changing $X$ can alter the decision $Y$~\cite{wachter2017counterfactual, moraffah2020causal}.}
For a \emph{failed KS test}, where a reference set $R$ and a test set $T$ fail the KS test, a counterfactual explanation is a minimum subset $S$ of the test set $T$ such that removing the subset from $T$ reverses the failed KS test into a passed one, that is, $R$ and $T\setminus S$ pass the KS test.
\mc{Using counterfactual explanations to interpret failed KS tests helps users gain more insights into changes and differences behind the failed KS tests.}

\begin{example}[Motivation]\rm
\label{ex:motivation}
\mc{A public health officer may want to compare the distributions of COVID-19 cases reported in August and September 2020 in the province of British Columbia, Canada.  She may use the cases in August as the reference set and those in September as the test set, where each COVID-19 reported case is associated with the age group of the patient. The cases are divided into $10$ age groups, (0-10), (10-19), $\ldots$, (80-89), and (90+).
A failed KS test between the two sets suggests that the infected cases in those two months unlikely follow the same distribution on age groups.  This information may be helpful to review the infection control policies.  As a KS test may raise false alarms~\cite{polyzotis2019data}, the officer may want to have a counterfactual explanation on this failed KS test, which reveals the cases that are likely relevant to the change. 
Example~\ref{ex:lexicographic_order} and our case study in Section~\ref{subsec:exp_effective} illustrate how such counterfactual explanations can help us to understand the changes.}
\end{example}

Although interpreting failed KS tests is interesting and has many potential applications, it has not been touched in literature. Although there are many counterfactual explanation methods interpreting the decisions of machine learning models~\cite{fong2017interpretable, akula2020cocox, grace}, unfortunately, the existing methods cannot be adopted to interpret failed KS tests. As reviewed in Section~\ref{sec:related_work}, to interpret a failed KS test, the existing methods have to solve an $L_{0}$-norm optimization problem, which is NP-hard~\cite{modas2019sparsefool}. Some methods~\cite{rabanser2019failing, pinto2019automatic} try to select the outliers in the test set as a hint to a failed KS test. However, because outlier detection methods and KS tests use different mechanisms to detect anomalies, there is no guarantee that the outliers are relevant to the failure of a KS test. 
Moreover, due to the Roshomon effect~\cite{moraffah2020causal}, multiple counterfactuals may co-exist for a failed KS test but not all of them are comprehensible to users~\cite{moraffah2020causal}. Simply presenting all counterfactuals not only may overwhelm users but is also computationally expensive~\cite{molnar2019}. 

A desirable idea is to find a counterfactual explanation that is most consistent with a user's domain knowledge so that the explanation is best comprehensible to the user~\cite{carvalho2019machine, mothilal2020explaining}. However, none of the existing methods can find such most comprehensible explanations.
Moreover, finding the most comprehensible explanation is far from trivial. A brute force method has to enumerate all subsets of a test set and, for each subset, conduct a KS test.  Thus, the brute force method takes exponential time. 

In this paper, we tackle the novel problem of producing counterfactual explanations for failed KS tests. 
We make several contributions.  Concept-wise, we propose the notion of comprehensible counterfactual explanations. Given a failed KS test, we find a smallest subset of the test set such that removing the subset from the test set reverses the failed KS test into a passed one. To address user comprehensibility~\cite{moraffah2020causal, carvalho2019machine}, we take a user's domain knowledge represented as a preference order on the data points in the test set, and guarantee to find the counterfactual explanation that is most consistent with the preference.
Computation-wise, we develop MOCHE (for \underline{MO}st \underline{C}ompre\underline{H}ensible \underline{E}xplanation), a two-step fast method that guarantees to find the most comprehensible counterfactual explanation on a failed KS test. Specifically, MOCHE first identifies the number of data points in the explanation and then efficiently constructs the most comprehensible explanation. We establish an important insight that the size of removed data points is the smallest integer satisfying a group of inequalities. Leveraging this property, an efficient searching algorithm is designed to find the explanation size. Then, MOCHE efficiently constructs the most comprehensible explanation by one scan of the data points in the test set. Experiment-wise, we conduct a systematic empirical study on a series of benchmark real datasets to verify the effectiveness, efficiency and scalability of most comprehensible counterfactual explanations and MOCHE.

\vspace{-1mm}

\section{Related Work}
\label{sec:related_work}

To the best of our knowledge, interpreting a failed KS test is a novel task that has not been systematically investigated in literature. Our study is broadly related to the Kolmogorov-Smirnov test~\cite{kifer2004detecting,schelter2020learning, keller2012hics}, counterfactual explanations~\cite{akula2020cocox, fong2017interpretable, grace}, adversarial attacks~\cite{papernot2017practical, brendel2017decision, croce2019sparse}  and outlier detection~\cite{ramaswamy2000efficient, gu2019statistical, ding2013anomaly}.

The Kolmogorov-Smirnov (KS) test~\cite{klotz1967asymptotic}  is a well-known statistical hypothesis test that checks whether two samples are originated from the same probability distribution. With the advantages of being efficient, non-parametric, and distribution-free~\cite{lall2015data}, the KS test has been widely used in many applications to detect differences, changes and abnormalities~\cite{ding2013anomaly}, such as identifying change points in time series~\cite{kifer2004detecting, ding2013anomaly}, maintaining machine learning models~\cite{schelter2020learning,rabanser2019failing, dos2016fast}, ensuring quality of encrypted or anonymized data~\cite{agrawal2004order, hay2008resisting}, and protecting databases from intrusion attacks~\cite{santos2014approaches}.

As illustrated in Section~\ref{sec:intro} and further elaborated later, understanding why a KS test is failed may be important in real world applications~\cite{pinto2019automatic, rabanser2019failing}. However, a failed KS test itself does not provide any hints on which data points in the test set may be related to the failure. Therefore, finding explanations of failed KS tests is a natural next step.


Counterfactual explanations~\cite{wachter2017counterfactual, sokol2019counterfactual, moraffah2020causal} have been widely adopted to interpret algorithmic decisions made in many real world applications~\cite{akula2020cocox, fong2017interpretable, grace}. 
Those methods~\cite{fong2017interpretable, akula2020cocox, van2019interpretable, moraffah2020causal} interpret a prediction on a given instance by applying small and interpretable perturbations on the instance such that the prediction is changed~\cite{moraffah2020causal}. For example, Fong~\textit{et~al.}~\cite{fong2017interpretable} interpret the prediction  of an image by finding the smallest pixel-deletion mask that leads to the most significant drop of the prediction score. As an extension, 
Akula~\textit{et~al.}~\cite{akula2020cocox} identify meaningful image patches that need to be added to or deleted from an input image. Van Looveren~\textit{et~al.}~\cite{van2019interpretable} use class prototypes to generate counterfactuals that lie close to the classifier's training data distribution.
Le~\textit{et~al.}~\cite{grace} use an entropy-based feature selection approach to limit the features to be perturbed.

Unfortunately, the existing counterfactual explanation methods cannot effectively and efficiently interpret a failed KS test by perturbing the data points in the test set. This is because, to minimize the number of perturbed data points, the existing methods need to minimize the $L_{0}$-norm of their perturbations~\cite{modas2019sparsefool}. However, such an optimization problem is NP-hard~\cite{nikolova2013description, modas2019sparsefool}. The existing methods cannot guarantee to reach a global minimum for the optimization problem in an efficient manner.

One may think adversarial attack methods~\cite{cheng2018query, brendel2017decision, croce2019sparse, papernot2017practical, papernot2016transferability} may be extended to find counterfactual explanations on failed KS tests. To attack a target classifier, an adversarial attack method generates an imperceptible perturbation on an input so that the prediction on the input is changed. Brendel~\textit{et~al.}~\cite{brendel2017decision} propose to generate adversarial perturbations by moving instances towards the estimated decision boundaries of a target model. Cheng~\textit{et~al.}~\cite{cheng2018query} formulate the black-box attack as an optimization problem, which can be solved by the zeroth order optimization approaches. Croce~\textit{et~al.}~\cite{croce2019sparse} propose to attack image classifiers by applying randomly selected one-pixel modifications on images.  One may generate counterfactual explanations on a failed KS test by attacking the KS test, that is, the perturbed data points can serve as a counterfactual explanation on the KS test. However, extending the existing adversarial attack methods to interpret failed KS tests also needs to minimize the $L_{0}$-norm of the perturbations and leads to the same computational challenge.

Outlier detection methods aim to detect samples that are different from the majority of the given data~\cite{aggarwal2015outlier}, such as distance-based approaches~\cite{ramaswamy2000efficient, gu2019statistical, angiulli2002fast, boniol2020series2graph}, density-based approaches~\cite{breunig2000lof, papadimitriou2003loci, goldstein2012histogram, 10.5555/1182635.1164145} and ensemble-based approaches~\cite{ding2013anomaly, lazarevic2005feature}. In general, outliers are regarded as abnormal data points~\cite{aggarwal2015outlier}.

Even though both the KS test and outlier detection methods can detect anomalies in data, the detected outliers in the test set cannot be used as a counterfactual explanation on a failed KS test. This is because outlier detection methods and the KS test use different mechanisms to detect anomalies. Different from the KS test, the outlier detection methods do not compare the distributions of the reference set and the test set. Therefore, there is no guarantee that outliers can explain a failed KS test.  Just removing the outliers cannot guarantee to reverse a failed KS test to a passed one. 


\section{Problem Formulation and Analysis}
\label{sec:problem_formulation}

In this section, we first review the basics of the Kolmogorov-Smirnov (KS) test.
Then, we investigate how to generate a counterfactual explanation on a KS test. 
Third, we discuss the comprehensibility of explanations, and formalize the problem of finding the most comprehensible explanation on a failed KS test. Next, we investigate the existence and uniqueness of most comprehensible counterfactual explanations.  Last, we describe a brute force method.

\subsection{The Kolmogorov-Smirnov Test}

Denote by $R=\{r_1, \dots, r_n\}$ a multi-set of real numbers from an unknown univariate probability distribution,  and by $T=\{t_1, \dots, t_m\}$ another multi-set of real numbers that are sampled from a distribution that may or may not be the same as $R$. We call $R$ a \emph{reference set} and $T$ a \emph{test set}. In this paper, by default multi-set is used.  In the rest of the paper, we use the terms ``set'' and ``multi-set'' interchangeably unless specifically mentioned.


The Kolmogorov-Smirnov (KS) test checks whether $T$ is sampled from the same probability distribution as  $R$ by comparing the empirical cumulative functions of $R$ and $T$.
In the KS test, the null hypothesis is that  $T$ is sampled from the same probability distribution as $R$. 

Conducting the KS test consists of $3$ steps as follows. 

\emph{Step 1.} We compute the \emph{KS statistic}~\cite{dos2016fast} by 
\begin{equation}
    \label{eq:ks_statistic}
    D(R, T)=
    \max_{x \in R \cup T}|F_{R}(x) - F_{T}(x)|,
\end{equation}
where $F_{R}(x)$ and $F_{T}(x)$ are the empirical cumulative functions of  $R$ and $T$, respectively.
Here, a larger value of $D(R, T)$ indicates that the empirical cumulative functions of  $R$ and $T$ are more different from each other.

\emph{Step 2.} For a user-specified significance level $\alpha$, we compute the corresponding target $p$-value~\cite{dos2016fast} by 
$    p = c_{\alpha} \sqrt{\frac{n+m}{n*m}}$,
where $c_{\alpha}=\sqrt{-\frac{1}{2}\ln(\frac{\alpha}{2})}$ is the critical value at significance level $\alpha$, $n=|R|$ is the number of data points in $R$, and $m=|T|$.

\emph{Step 3.} We compare $p$ and $D(R, T)$. If $D(R, T) > p$, we reject the null hypothesis at significance level $\alpha$. This means the empirical cumulative functions $F_R(x)$ and $F_T(x)$ are significantly different from each other, and thus it is unlikely $T$ is sampled from the same distribution as  $R$.
If $D(R, T) \leq p$, we cannot reject the null hypothesis at significance level $\alpha$. There is not enough evidence showing that $T$ is not sampled from the same distribution as $R$.

If the null hypothesis is rejected by the KS test, we say $R$ and $T$ \emph{fail} the KS test and it is a \emph{failed KS test}. Otherwise, we say $R$ and $T$ \emph{pass} the KS test.

To compute the KS statistic between $R$ and $T$, we need to sort the elements in $R \cup T$ in ascending order. Therefore, it takes $O((n+m)\log(n+m))$ time to conduct the KS test.

\nop{
Table~\ref{tbl:notations} summarizes some frequently used notations in this paper.

\begin{table}[t]
\scalebox{0.90}{
\begin{tabular}{|l|l|}
\hline\small
\textbf{Notation}      & \textbf{Description} \\ \hline
$R$           & The reference set of a KS test.          \\ \hline
$T$           & The test set of a KS test.           \\ \hline
$S$           & An $h$-subset of $T$.                         \\ \hline
$L$           & A preference list on the test set $T$            \\ \hline
$\prec$ & The lexicographical order based on $L$ \\ \hline
$\mathbf{C}$           & An $h$-cumulative vector.            \\ \hline
$\mathcal{I}$ & An explanation.            \\ \hline
$q$           & The number of unique data points in $R \cup T$.   \\ \hline
$n$           & The size of $R$, that is $n=|R|$.            \\ \hline
$m$           & The size of $T$, that is $m=|T|$.              \\ \hline
$k$           & The size of $\mathcal{I}$, that is $k=|\mathcal{I}|$.            \\ \hline
$\mathbf{C}[i]$, $c_i$         & The $i$-th element of an $h$-cumulative vector $\mathbf{C}$. \\ \hline
$l_i^h$       & The lower bound of $\mathbf{C}[i]$.            \\ \hline
$u_i^h$       & The upper bound of $\mathbf{C}[i]$.            \\ \hline
$x_i$         & An element in $R \cup T$.            \\ \hline
$c_{\alpha}$ & The critical value at significance level $\alpha$  \\ \hline
\end{tabular}
}
\caption{Frequently used notations.}
\label{tbl:notations}
\end{table}
}

\subsection{Counterfactual Explanations on the KS Test}\label{sec:failedtest}

Why are we interested in failed KS tests?
More often than not, a failed hypothesis test indicates something unusual or unexpected~\cite{chen2019novel,lall2015data,rabanser2019failing}.
\mc{As many important decisions are made based on failed KS tests~\cite{yu2018request, kifer2004detecting}, 
it is important to interpret a failed KS test so that we can make better responses to the change and abnormality alarms.}

\mc{Counterfactual explanation is an explanation technique proposed by the community of explainable artificial intelligence. It has been well demonstrated to be more human-friendly than other types of explanations~\cite{sokol2019counterfactual, moraffah2020causal}. 
The counterfactual explanation methods interpret a decision $Y$ by finding a smallest set of relevant factors $X$, such that changing $X$ can alter the decision $Y$~\cite{wachter2017counterfactual, moraffah2020causal}. The set of factors $X$ is called a counterfactual explanation on $Y$.}
Following the above principled idea, we have the following definition.

\begin{definition}
\label{def:interpretation}
For a reference set $R$ and a test set $T$ that fail the KS test at a significance level $\alpha$,
a \textbf{counterfactual explanation} on the failed KS test is a smallest subset $\mathcal{I}$ of the test set $T$, such that $R$ and $T\setminus \mathcal{I}$ pass the KS test at the same significance level $\alpha$. 
\end{definition}

A counterfactual explanation is also called an \emph{explanation} for short when the context is clear.

\subsection{The Most Comprehensible Counterfactual Explanation on a KS Test}
\label{sec:most_comprehensible_interpretation}

Like many previously proposed counterfactual explanations~\cite{wachter2017counterfactual, fong2017interpretable, mothilal2020explaining}, the counterfactual explanations on a failed KS test suffer from the \emph{Roshomon effect}~\cite{moraffah2020causal}, that is, the number of unique counterfactual explanations on a failed KS test can be as large as ${|T| \choose |\mathcal{I}|}$. 
Simply presenting all counterfactuals not only may overwhelm users but is also computationally expensive~\cite{molnar2019}. 

As discovered by many studies on counterfactual explanations~\cite{mothilal2020explaining, kusner2017counterfactual, sokol2019counterfactual}, not all counterfactual explanations are equally comprehensible to a user. 
Due to the effect of confirmation bias~\cite{nickerson1998confirmation}, an explanation is more comprehensible if it is more consistent with the user's domain knowledge~\cite{miller2019explanation}. 
As a result, a typical way to overcome the Roshomon effect is to rank all explanations according to the user's preference based on the domain knowledge, and return the most preferred explanation to the user~\cite{mothilal2020explaining, artelt2019efficient}.

Following the above idea, we model a user's preference as a total order on the data points in the test set $T$, that is, a \emph{preference list} $L$ on the test set $T$. Each data point has a unique rank in $L$. The data points having smaller ranks in $L$ are more preferred by the user. 

A typical task of recommendation system is to recommend a group of items to a user, such that the group best satisfies the user's preference~\cite{xie2010breaking}. The existing studies~\cite{tschiatschek2017selecting, chen2015information, benouaret2019efficient, xie2010breaking} discover that a user's interest in a group is dominated by the user's top favorite items in the group. In the same vein, one can think of an explanation as a recommended group of data points. Given two explanations $\mathcal{I}_1$ and $\mathcal{I}_2$ on a failed KS test, if $\mathcal{I}_1$ includes better-ranked data points in $L$ than $\mathcal{I}_2$ does, $\mathcal{I}_1$ is more preferred by the user than $\mathcal{I}_2$, and thus is more comprehensible.

Based on the above intuition, an explanation with a smaller lexicographical order\footnote{Given a total order $\prec_L$ on items, the lexicographical order $\prec_{\text{lexicographical}}$ is $x_1x_2 \cdots x_n \prec_{\text{lexicographical}} y_1 y_2 \cdots y_l$ if (1) $x_1 \prec_L y_1$; (2) there exists $i_0$ $(1 < i_0 \leq \min\{n, l\})$, $x_i = y_i$ for $1 \leq i < i_0$ and $x_{i_0} \prec_L y_{i_0}$; or (3) $m<l$ and for $1 \leq i \leq m$, $x_i=y_i$.  Lexicographical order is also known as dictionary order.
} based on the preference list $L$ is more preferred by the user. Specifically, for two explanations $\mathcal{I}_1$ and $\mathcal{I}_2$, $\mathcal{I}_1 \subseteq T$, $\mathcal{I}_2 \subseteq T$ and $|\mathcal{I}_1|=|\mathcal{I}_2|$, 
we sort the data points in $\mathcal{I}_1$ and $\mathcal{I}_2$ in the order of $L$. Denote by $\mathcal{I}[i]$ the $i$-th data point in $\mathcal{I}$ in the order of $L$. Let $i_0>0$ be the smallest integer such that $\mathcal{I}_1[i_0] \neq \mathcal{I}_2[i_0]$. $\mathcal{I}_1$ precedes $\mathcal{I}_2$ in the lexicographical order, denoted by $\mathcal{I}_1 \prec \mathcal{I}_2$, if $\mathcal{I}_1[i_0]$ precedes $\mathcal{I}_2[i_0]$ in $L$. If $\mathcal{I}_1 \prec \mathcal{I}_2$, $\mathcal{I}_1$ includes more top-ranked data items in $L$ than $\mathcal{I}_2$. 

\begin{definition}
\label{def:most_comprehensible_interpretation}
Given a failed KS test and a preference list $L$, the \textbf{most comprehensible counterfactual explanation} is the explanation that has the smallest lexicographical order based on $L$.
\end{definition}

The notion of comprehensible explanation captures user preferences in domain knowledge. In order to capture different domain knowledge, we can employ different preference lists to sort the data points in the test set. 

\begin{figure}
    \centering
    \includegraphics[page=1, width=1\linewidth]{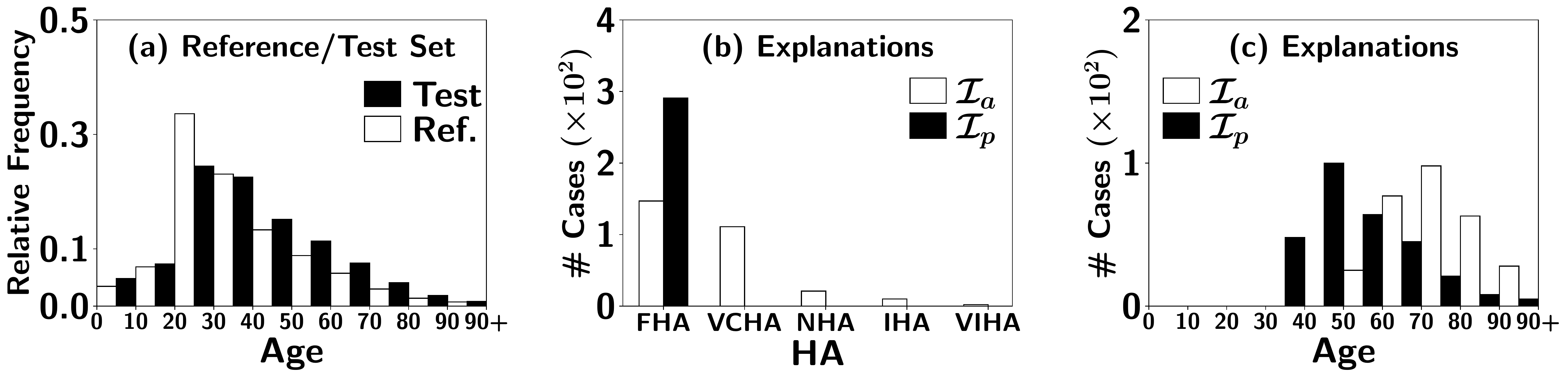}
    \caption{\mc{The histograms of (a) the reference set and the test set; (b) the distributions of the explanations $\mathcal{I}_a$ and $\mathcal{I}_p$ in Example~\ref{ex:lexicographic_order} on HAs; and (c) those on age groups.}}
\label{fig:example_covid}
\end{figure}

\begin{example}\rm
\label{ex:lexicographic_order}
\mc{Let us consider the KS test conducted on the COVID-19 cases discussed in Example~\ref{ex:motivation}. 
The reference set (August) and the test set (September) have 2,175 and 3,375 data points, respectively.
The histograms of the two sets are shown in Figure~\ref{fig:example_covid}a.
Each bin on the X-axis represents an age group. Please refer to Section~\ref{sec:settings} for more details about the dataset.
The two sets fail the KS test with significance level $\alpha=0.05$. }

\mc{Since COVID-19 may be more contagious in regions of larger population, a public health officer may sort the reported cases into a preference list $L_p$ in population descending order of the 
reported health authority (HA for short). Please see \url{https://catalogue.data.gov.bc.ca/dataset/health-authority-boundaries} for details of HA in British Columbia.  The data points from HAs with large population are ranked higher, while the cases from the same HAs are sorted arbitrarily. }

\mc{COVID-19 is also known to hit seniors harder.  Alternatively, the officer may sort the reported cases into a preference list $L_a$ in age group descending order. People in more senior age groups are ranked higher, while the cases from the same age group are sorted arbitrarily.}

\mc{Given preference lists $L_p$ and $L_a$, our method MOCHE produces the corresponding most comprehensible counterfactual explanations $\mathcal{I}_p$ and $\mathcal{I}_a$, respectively. Figures~\ref{fig:example_covid}b and~\ref{fig:example_covid}c show the distributions of the two explanations on HAs and age groups, respectively. 
The X-axis of Figure~\ref{fig:example_covid}b shows the HA ids  from left to right in population descending order.
Both $\mathcal{I}_a$ and  $\mathcal{I}_p$ include 291 data points.}

\mc{As shown in Figure~\ref{fig:example_covid}b, all data points in $\mathcal{I}_p$ are from FHA (Fraser HA), the HA with the largest population. Based on $L_{p}$, we have $\mathcal{I}_p \prec \mathcal{I}_a$ in lexicographical order, that is, $\mathcal{I}_p$ is more preferable than $\mathcal{I}_a$ when HAs of large population are concerned.
As shown in Figure~\ref{fig:example_covid}c, $\mathcal{I}_a$ contains more senior people. Based on $L_{a}$, we have $\mathcal{I}_a \prec \mathcal{I}_p$ in lexicographical order, that is, $\mathcal{I}_a$ is more preferable when senior people are more concerned.} 
\end{example}

\subsection{Existence  and Uniqueness}

For a failed KS test at significance level $\alpha$, our task is to find the most comprehensible counterfactual explanation $\mathcal{I}$ on the KS test. Does such an explanation always exist? If so, is the most comprehensible counterfactual explanation unique?

\begin{proposition}
\label{proposition:unique_solution}
When the significance level $\alpha \leq \frac{2}{e^2}$, there exists a unique most comprehensible explanation on a failed KS test.
\begin{proof}

(Existence) Consider a subset $S \subset T$, where $|S|=|T|-1$. $|T \setminus S|=1$.
The $p$-value of the KS test between $R$ and $T \setminus S$ is $p=c_{\alpha}\sqrt{\frac{n + 1}{n * 1}}$, where $n$ is the size of $R$ and $c_{\alpha}=\sqrt{-\ln(\frac{\alpha}{2})*\frac{1}{2}}$. Since $\alpha \leq \frac{2}{e^2}$, we have $c_{\alpha} \geq 1$ and $p \geq \sqrt{\frac{n+1}{n}}\geq 1$. Since $D(R, T \setminus S)$ is the absolute difference between the two empirical cumulative functions, $1 \geq D(R, T \setminus S)$. Therefore, $p \geq D(R, T \setminus S)$.  That is, $R$ and $T \setminus S$ pass the KS test.
Since an explanation is a smallest subset that reverses a failed KS test to a passed one, there must exist an explanation on a failed KS test given $R$ and $T \setminus S$ passing the test.

(Uniqueness) Since each data point has a unique rank in $L$, two distinct explanations cannot be equivalent in the lexicographical order. Thus, the most comprehensible explanation is unique.
\end{proof}
\end{proposition}

Statistical tests in practice typically use a significance level of $0.05$ or lower.  $\frac{2}{e^2} > 0.27$, which is far over the range of significance levels used in statistical tests.  Therefore, our problem formulation is practical and guarantees a unique solution in practice. 
\subsection{A Brute Force Method}
\label{sec:baseline}


A na\"ive method to find the most comprehensible explanation is to enumerate all subsets of the test set $T$ and check against Definition~\ref{def:most_comprehensible_interpretation}.
This brute-force method checks an exponential number of subsets, which is prohibitive for large test sets. 

Even in a brute force method, we can significantly reduce the number of subsets that need to be checked by early pruning a large number of subsets.  According to Definitions~\ref{def:interpretation} and~\ref{def:most_comprehensible_interpretation}, we can sort all subsets of $T$ first by the size, from small to large, and then by the lexicographical order. This can be done by a breadth-first traversal of a set enumeration tree~\cite{10.5555/3087223.3087278}. The first subset $S$ in this order such that $R$ and $T\setminus S$ pass the KS test is the most comprehensible explanation. 
\section{Searching for Explanation Size}
\label{sec:interpretation_size}

In this section, we first describe the two-phase framework of MOCHE (for \underline{MO}st \underline{C}ompre\underline{H}ensible \underline{E}xplanation), a fast method to find the most comprehensible counterfactual explanations. Then, we thoroughly explore how to compute the size of explanations fast.

\subsection{MOCHE}

According to Definition~\ref{def:interpretation}, all explanations have the same size. Once we find an explanation $\mathcal{I}$, we can safely ignore all subsets of $T$ whose sizes are not equal to $|\mathcal{I}|$, no matter they can reverse the KS test or not.  Based on this idea, the MOCHE method proceeds in two phases.  In phase 1, MOCHE tries to find the size of explanations.  In phase 2, MOCHE tries to identify the most comprehensible explanations, that is, the smallest one in lexicographical order.

A subset $S$ of $T$ such that $|S|=h$ is called an \emph{$h$-subset}. An $h$-subset $S$ is a \emph{qualified h-subset} if $R$ and $T \setminus S$ pass the KS test.
The first bottleneck is to check, for a given $h>0$, whether there exists a qualified $h$-subset $S$. 
A brute-force implementation has to conduct the KS test a large number of times on all $h$-subsets. The time complexity is $O({m \choose h}  (m+n-h)\log(n+m-h))$.

Our first major technical result in this section is that checking the existence of a qualified $h$-subset does not have to conduct the KS test on all $h$-subsets. 
With a carefully designed data structure named cumulative vector to represent an $h$-subset of $T$, we establish a fast verification method for qualified cumulative vectors. Checking the existence of a qualified $h$-cumulative vector and thus a qualified $h$-subset only takes $O(m+n)$ time.

The second bottleneck is to find the size of explanations efficiently.  A brute-force method has to search from $1$ to $m-1$ one by one and, for each size $h$, check the $h$-subsets. The second
major technical result in this section tackles this bottleneck by deriving a lower bound $\hat{k}$ on $k$, the size of all explanations. This lower bound reduces the search range of $h$ from $[1, k]$ to $[\hat{k}, k]$, which further reduces the time complexity of phase 1 to $O((m+n)\log(m) + (k - \hat{k})(m+n))$.

\subsection{Cumulative Vectors}
\label{subsec:configuration_of_subset}

Essentially, the KS test compares the cumulative distribution functions of a reference set and a test set.  Since there are only finite numbers of data points in a reference set and a test set, we can represent the cumulative distribution function of a reference set, a test set or a subset of the test set using a sequence of the values of the cumulative distribution function at the data points appearing at either the reference set or the test set.  This observation motivates the design of the cumulative vectors.


We make a \emph{base vector} $\mathbf{V}=\langle x_1, \ldots, x_q\rangle$ from sets $R$ and $T$, such that $x_1, \ldots, x_q$ are the unique data points in $R \cup T$.  
No matter how many times $x_i$ appears in $R \cup T$, it only appears once in $\mathbf{V}$.
Thus, $q = \|R \cup T\|$, where $R$ and $T$ are treated as sets instead of multi-sets and $q$ is the cardinality of the union, that is, duplicate items are not double counted. The elements in $\mathbf{V}$ are sorted in the value ascending order, that is $x_1 < x_2 < \cdots < x_q$.


\begin{definition}
\label{def:configuration}
The \emph{cumulative vector} of an $h$-subset $S\subseteq T$ is a $(q+1)$-dimensional vector $\mathbf{C}_S=\langle c_0, c_1, \dots, c_q \rangle$, where $c_0=0$, and for $1 \leq i \leq q$, $c_i$ is the number of data points in $S$ that are smaller than or equal to $x_i$ in $\mathbf{V}$, that is $c_i = |\{x \in S \mid x \leq x_i\}|$.  We also write $c_i$ as $\mathbf{C}_S[i]$.
\end{definition}

\begin{example}\rm
\label{ex:config_def}
\mc{Consider a test set $T=\{13, 13, 12, 20\}$ and a reference set $R=\{14, 14,  14,  14, 20, 20, 20, 20\}$. The base vector $\mathbf{V}=\langle 12, 13, 14, 20 \rangle$. For a subset $S=\{13, 13\}$ of $T$, the cumulative vector is $\mathbf{C}_S=\langle 0, 0, 2, 2, 2\rangle$. }
\end{example}

According to Definition~\ref{def:configuration}, a cumulative vector $\mathbf{C}_S$ contains all information to derive the cumulative distribution function $F_{T\setminus S}$ straightforwardly.  For a cumulative vector $\mathbf{C}_S=\langle c_0, c_1, \dots, c_q \rangle$ and any $i$ $(1 \leq i \leq q)$, $c_i - c_{i-1}$ is the number of times that $x_i$ appears in $S$. Thus, the value of the empirical cumulative distribution function of $T \setminus S$ at $x_i$ can be computed by
$    F_{T \setminus S}(x_i)=\frac{\mathbf{C}_T[i] - c_i}{m - c_q}$, 
where $\mathbf{C}_T$ is the cumulative vector of $T$ and $\mathbf{C}_T[i]$ is the $i$-th element of $\mathbf{C}_T$ and thus is the number of data points in $T$ that are
not larger than
$x_i$.

Clearly, given a reference set $R$ and a test set $T$, every unique subset $S\subseteq T$ corresponds to a unique cumulative vector $\mathbf{C}_S$ and a unique cumulative distribution function $F_{T\setminus S}$, and vice versa.  Recall that, if a subset $T \setminus S$ and $R$ pass the KS test, $S$ is called a qualified $h$-subset, where $h=|S|$.  Correspondingly, we call the cumulative vector $\mathbf{C}_S$ a qualified $h$-cumulative vector.

\subsection{Existence of Qualified $h$-Cumulative Vectors}
\label{sec:fast_check_existence_of_interpretations}

For a given $h$ $(1\leq h \leq |T|)$, can we quickly determine whether there exists a qualified $h$-cumulative vector and thus a qualified $h$-subset?  Before we state the major result, we need the following.

\begin{lemma}\label{theorem:critical_set_c_i_range}
Given a reference set $R$ and a test set $T$, for $S \subset T$, $\mathbf{C}_S=\langle c_0, c_1, \ldots, c_q\rangle$ is a qualified cumulative vector if and only if, for each $i$ $(1 \leq i \leq q)$, the following two inequalities hold.
\begin{subequations}\small
\label{eq:config_of_critical_set_in_recursive}
\begin{align}
     \max(\lceil \Gamma(i, h) - \Omega(h) \rceil, h-m+\mathbf{C}_T[i], \mathbf{C}_{S}[i-1]) \leq \mathbf{C}_{S}[i] \label{eq:lower_range_of_c_i}\\
     \mathbf{C}_{S}[i] \leq \min(\lfloor \Gamma(i, h) + \Omega(h) \rfloor, \mathbf{C}_T[i] - \mathbf{C}_T[i-1] + \mathbf{C}_{S}[i-1], h)  \label{eq:upper_range_of_c_i}
\end{align}
\end{subequations}
where $\Omega(h) = c_{\alpha}\sqrt{m - h  + \frac{(m - h)^2}{n}}$, $\Gamma(i, h)=\mathbf{C}_T[i] - \frac{m - h}{n} \mathbf{C}_R[i]$, and $\mathbf{C}_R$ and $\mathbf{C}_T$ are the cumulative vectors of $R$ and $T$, respectively. 

\begin{proof}
(Necessity) 
According to the definition of KS statistic in Equation~\ref{eq:ks_statistic}, an $h$-subset $S$ is qualified if and only if $\forall i \; (1 \leq i \leq q)$, 
$    |F_{R}(x_i) - F_{T \setminus S}(x_i)| \leq c_{\alpha}\sqrt{\frac{n + m - h}{n*(m - h)}}$.
Since $F_{R}(x_i) = \frac{\mathbf{C}_{R}[i]}{n}$ and $F_{T \setminus S}(x_i) = \frac{\mathbf{C}_{T}[i] - \mathbf{C}_{S}[i]}{m - h}$, we have 
$    \Big | \frac{\mathbf{C}_{R}[i]}{n} - \frac{\mathbf{C}_{T}[i] - \mathbf{C}_{S}[i]}{m - h}\Big | \leq c_{\alpha}\sqrt{\frac{n+m-h}{n*(m - h)}}$.
After simplification, we have 
$        \Gamma(i, h) - \Omega(h)  \leq \mathbf{C}_{S}[i]
        \leq \Gamma(i, h) + \Omega(h)$.
Since $\mathbf{C}_{S}[i]$ is a non-negative integer, we immediately have
\begin{equation}
\label{eq:round}
    \lceil \Gamma(i, h) - \Omega(h) \rceil  \leq \mathbf{C}_{S}[i] \leq \lfloor  \Gamma(i, h) + \Omega(h) \rfloor.
\end{equation}
Since $h - \mathbf{C}_{S}[i]$ and $m - \mathbf{C}_{T}[i]$ are the numbers of data points in $S$ and $T$ that are larger than $x_i$, respectively, and $S \subset T$, $h - \mathbf{C}_{S}[i] \leq m - \mathbf{C}_{T}[i] $ holds, that is, $h - m + \mathbf{C}_{T}[i] \leq \mathbf{C}_{S}[i]$. Since $\mathbf{C}_{S}[i-1] \leq \mathbf{C}_{S}[i]$, Equation~\ref{eq:lower_range_of_c_i} holds.

Since $\mathbf{C}_{S}[i] - \mathbf{C}_{S}[i-1]$ and $\mathbf{C}_{T}[i] - \mathbf{C}_{T}[i-1]$ are the numbers of times $x_i$ appears in $S$ and $T$, respectively, and $S \subset T$, $\mathbf{C}_{S}[i] - \mathbf{C}_{S}[i-1] \leq \mathbf{C}_{T}[i] - \mathbf{C}_{T}[i-1]$, that is, $\mathbf{C}_{S}[i] \leq \mathbf{C}_{S}[i-1] + \mathbf{C}_{T}[i] - \mathbf{C}_{T}[i-1]$.
Using the righthand side of Equation~\ref{eq:round} and $\mathbf{C}_{S}[i] \leq h$ by definition, Equation~\ref{eq:upper_range_of_c_i} holds. 


(Sufficiency) For any $h$-cumulative vector $\mathbf{C}_{S}$ that satisfies Equations~\ref{eq:lower_range_of_c_i} and~\ref{eq:upper_range_of_c_i}, we construct a set $S$ such that for each $i$ $(1 \leq i \leq q)$, data point $x_i$ appears in $S$ $(\mathbf{C}_{S}[i]-\mathbf{C}_{S}[i-1])$ times. Since $\mathbf{C}_{S}[i]$ and $\mathbf{C}_{S}[i-1]$ satisfy the inequality $\mathbf{C}_{S}[i] - \mathbf{C}_{S}[i-1] \leq \mathbf{C}_{T}[i] - \mathbf{C}_{T}[i-1]$, the number of times $x_i$ appearing in $S$ is smaller than or equal to the number of times $x_i$ appearing in $T$. Plugging $\mathbf{C}_{T}[q] = m$ into Equation~\ref{eq:lower_range_of_c_i}, we have $h \leq \mathbf{C}_{S}[q]$. Since $\mathbf{C}_{S}[q]$ also satisfies Equation~\ref{eq:upper_range_of_c_i}, we have $h = \mathbf{C}_{S}[q]$. From the way that $S$ is constructed, we know that $S$ has $\mathbf{C}_{S}[q]$ elements. Therefore, $S$ is an $h$-subset of $T$.

We show that $S$ is a qualified $h$-subset. Since $\mathbf{C}_{S}$ satisfies Equation~\ref{eq:config_of_critical_set_in_recursive}, for each $i$ $(1 \leq i \leq q)$, $\mathbf{C}_{S}[i]$ satisfies Equation~\ref{eq:round}. Plugging Equation~\ref{eq:round} into $F_{T \setminus S}$, we have $\forall i (1 \leq i \leq q)$, $|F_{R}(x_i) - F_{T \setminus S}(x_i)| \leq c_{\alpha}\sqrt{\frac{n+m-h}{n*(m - h)}}$. This means $R$ and $T \setminus S$ can pass the KS test, thus $S$ is a qualified $h$-subset of $T$ and $\mathbf{C}_{S}$ is a qualified $h$-cumulative vector. The sufficiency follows.
\end{proof}
\end{lemma}





Lemma~\ref{theorem:critical_set_c_i_range} transforms conducting the KS test to checking Equations~\ref{eq:lower_range_of_c_i} and~\ref{eq:upper_range_of_c_i}.
Given $h$ $(1 \leq h \leq m-1)$, Equations~\ref{eq:lower_range_of_c_i} and~\ref{eq:upper_range_of_c_i} recursively give a lower bound and an upper bound of each element $\mathbf{C}[i]$ $(1 \leq i \leq q)$ of an $h$-cumulative vector $\mathbf{C}$,  respectively.
The lower bound and the upper bound of $\mathbf{C}[i]$ depend on the lower bound and the upper bound of $\mathbf{C}[i-1]$, respectively. 

Denote by $l_i^h$ and $u_i^h$ the lower bound and the upper bound of $\mathbf{C}[i]$ in any qualified $h$-cumulative vector $\mathbf{C}$. We compute $l_1^h$ and $u_1^h$ by plugging $\mathbf{C}[0] = 0$ into Equations~\ref{eq:lower_range_of_c_i} and~\ref{eq:upper_range_of_c_i}. Then, we plug $\mathbf{C}[1] = l_1^h$ into Equation~\ref{eq:lower_range_of_c_i} and $\mathbf{C}[1] = u_1^h$ into Equation~\ref{eq:upper_range_of_c_i} to compute the lower bound $l_2^h$ and the upper bound $u_2^h$ of $\mathbf{C}[2]$, respectively. By iteratively plugging $\mathbf{C}[i-1] = l^h_{i-1}$ into Equation~\ref{eq:lower_range_of_c_i} and $\mathbf{C}[i-1] = u^h_{i-1}$ into Equation~\ref{eq:upper_range_of_c_i}, we can compute the lower bound and the upper bound of every $\mathbf{C}[i]$ of qualified $h$-cumulative vectors $\mathbf{C}$. The closed form formulae of $l_i^h$ and $u_i^h$ $(1 \leq i \leq q)$ are 
\begin{subequations}
\begin{align}
    l^h_i &= \max(\lceil M(i, h) - \Omega(h) \rceil, h - m + \mathbf{C}_{T}[i], 0)  \label{eq:unroll_bounds_of_c_i_lower}\\
    u^h_i &= \min(\lfloor \Gamma(i, h) + \Omega(h) \rfloor, \mathbf{C}_{T}[i], h) \label{eq:unroll_bounds_of_c_i_upper}
\end{align}
\end{subequations}where $M(i, h) = \max_{j=1}^i\{ \Gamma(j, h)\}$. We define $l_0^h = u_0^h = 0$, as $\mathbf{C}[0] = 0$ is a constant.

\mc{Given the lower bounds $l_i^h$ and the upper bounds $u_i^h$ of the element in any qualified $h$-cumulative vectors, if for each $i$ $(1 \leq i \leq q)$, $l^h_i \leq u^h_i$, we can construct an $h$-cumulative vector $\mathbf{C}$ by selecting each element $\mathbf{C}[i]$ from $[l^h_i, u^h_i]$. Based on this intuition,
we use the lower bounds and the upper bounds of $\mathbf{C}[1], \mathbf{C}[2], \ldots, \mathbf{C}[q]$ to derive a sufficient and necessary condition for the existence of a qualified $h$-cumulative vector $\mathbf{C}$ as follows.}

\begin{theorem}
\label{theorem:l_u_bound_fast_check}
Given the KS test with a reference set $R$ and a test set $T$, for $h$ $(1 \leq h \leq m-1)$, there exists a qualified $h$-cumulative vector if and only if for each $i$ $(1 \leq i \leq q)$, $l^h_i \leq u^h_i$.
\end{theorem}

\begin{proof}

(Necessity) Since $l_i^h$ and $u_i^h$ are the lower bound and the upper bound of $\mathbf{C}[i]$, respectively, the necessity is straightforward.

(Sufficiency) Assuming for each $i$ $(1 \leq i \leq q)$, $l_{i}^h \leq u_{i}^h$, we construct a qualified $h$-cumulative vector $\mathbf{C}$ as follows. We start by setting $\mathbf{C}[q] = u_q$, and then for $i$ iterating from $q$ to $1$, we choose an integer $\mathbf{C}[i-1]$ from $[l_{i-1}^h, u_{i-1}^h]$, such that 
$    0 \leq \mathbf{C}[i] - \mathbf{C}[i-1] \leq \mathbf{C}_{T}[i] - \mathbf{C}_{T}[i-1]$.

Now we show that such an integer $\mathbf{C}[i-1]$ always exists. Since $l_i^h$ is derived by setting $\mathbf{C}[i-1] = l_{i-1}^h$ in Equation~\ref{eq:lower_range_of_c_i} and $u_{i}^h$ is derived by setting $\mathbf{C}[i-1] = u_{i-1}^h$ in Equation~\ref{eq:upper_range_of_c_i}, we have $l^h_{i-1} \leq l^h_{i}$ and  $u^h_{i} \leq u^h_{i-1} + \mathbf{C}_{T}[i] - \mathbf{C}_{T}[i-1]$.
Since $l_i^h \leq \mathbf{C}[i] \leq u^h_i$, we have 
$    l^h_{i-1} \leq \mathbf{C}[i] \leq u^h_{i-1} + \mathbf{C}_{T}[i] - \mathbf{C}_{T}[i-1]$.
Since $l^h_{i-1} \leq u^h_{i-1}$ and they are integers, there exists an integer $\mathbf{C}[i-1] \in [l^h_{i-1}, u^h_{i-1}]$, such that $0 \leq \mathbf{C}[i] - \mathbf{C}[i-1] \leq \mathbf{C}_{T}[i] - \mathbf{C}_{T}[i-1]$. Thus, an $h$-cumulative vector $\mathbf{C}$ can be constructed by  iteratively applying the above operations to set up elements in $\mathbf{C}$. 

Last, we prove that $\mathbf{C}$ is a qualified $h$-cumulative vector by showing that for each $i$ $(1 \leq i \leq q)$, $\mathbf{C}[i]$ satisfies Equations~\ref{eq:lower_range_of_c_i} and \ref{eq:upper_range_of_c_i}. 
According to the definition of an $h$-cumulative vector, we have $\mathbf{C}[i-1] \leq \mathbf{C}[i]$. 
By Equation~\ref{eq:unroll_bounds_of_c_i_lower}, we have $h - m + \mathbf{C}_{T}[i] \leq l_i^h$ and $\lceil \Gamma(i, h) - \Omega(h) \rceil \leq l_i^h$. 
Since $l_i^h \leq \mathbf{C}[i]$, $\mathbf{C}[i]$ satisfies Equation~\ref{eq:lower_range_of_c_i}.  By Equation~\ref{eq:unroll_bounds_of_c_i_upper}, we have $u_i^h \leq h$ and $u_i^h \leq \lfloor \Gamma(i, h) + \Omega(h) \rfloor$. According to how $\mathbf{C}[i-1]$ is selected, $\mathbf{C}[i]$ and $\mathbf{C}[i-1]$ satisfy $\mathbf{C}[i] - \mathbf{C}[i-1] \leq \mathbf{C}_{T}[i] - \mathbf{C}_{T}[i-1]$. Since $\mathbf{C}[i] \leq u_i^h$, $\mathbf{C}[i]$ satisfies Equation~\ref{eq:upper_range_of_c_i}. The sufficiency follows Lemma~\ref{theorem:critical_set_c_i_range} immediately.
\end{proof}

According to Theorem~\ref{theorem:l_u_bound_fast_check}, we can efficiently check the existence of a qualified $h$-cumulative vector by checking the $q$ pairs of lower bounds and upper bounds, $(l_1^h, u_1^h), \ldots, (l_q^h, u_q^h)$. Each pair of bounds can be computed and checked in $O(1)$ time. 
Since $q \leq n + m$, the time complexity of checking the existence of a qualified $h$-cumulative vector is $O(n+m)$. 

Since the existence of a qualified $h$-cumulative vector is equivalent to the existence of a qualified $h$-subset, 
we can tackle the first efficiency bottleneck by checking the $q$ pairs of lower bounds and upper bounds. 
This reduces the time complexity of checking the existence of a qualified $h$-subset from $O({m \choose h}(m + n - h)\log(m + n - h))$ to $O(n + m)$. 

To find the size of explanations, for each subset size $h$ $(1 \leq h \leq m-1)$, we need to apply Theorem~\ref{theorem:l_u_bound_fast_check} to check the existence of a qualified $h$-cumulative vector. Therefore, the overall time complexity of finding the size of explanations is $O(m(m+n))$. Next, we further reduce the time complexity to $O((m+n) \log(m) + (m+n)(k - \hat{k}))$, where $\hat{k}$ is a lower bound on the size of explanations $k$.

\begin{example}\rm
\label{ex:bounds}
\mc{One can verify that the reference set and the test set in Example~\ref{ex:config_def} fail the KS test with significance level $\alpha = 0.3$. When $h=1$, the lower bound $l_2^h = 2$ and the upper bound $u_2^h =1$. As $l_2^h > u_2^h$, by Theorem~\ref{theorem:l_u_bound_fast_check}, there does not exist a qualified $1$-cumulative vector. When $h=2$, $(l_1^h, u_1^h)=(0, 1)$,  $(l_2^h, u_2^h)=(1, 2)$, $(l_3^h, u_3^h)=(1, 2)$ and $(l_4^h, u_4^h)=(1, 2)$. By Theorem~\ref{theorem:l_u_bound_fast_check}, there exists a qualified $2$-cumulative vector and thus a qualified $2$-subset. Since the smallest size of a qualified subset is $2$, the explanation size $k=2$.}
\end{example}

\subsection{Finding a Lower Bound on Explanation Size by Binary Search}
\label{subsec:binary_search}

To tackle the second efficiency bottleneck, in this subsection, we develop a technique to find a lower bound on the size of explanations in $O((m+n) \log m)$ time. Using this technique, to find the size of explanations, we only need to check the subset sizes that are larger than or equal to the lower bound. 

To reduce the number of subset sizes to be checked, we develop a necessary condition for the existence of a qualified $h$-cumulative vector with respect to $h$. \mc{The necessary condition is obtained by relaxing the sufficient and necessary condition stated in Theorem~\ref{theorem:l_u_bound_fast_check}.} The necessary condition has a nice monotonicity with respect to $h$. If an integer $h$ $(1 \leq h \leq m-2)$ satisfies the condition,  all integers from $h+1$ to $m-1$ also satisfy the necessary condition.
\mc{This is because the right hand side of each inequality in Equation~\ref{eq:binary_search_eq_sys} increases faster than its left hand side as $h$ increases.} Thus, we can leverage this property to find a lower bound $\hat{k}$ of the explanation size $k$ by a binary search in $O((m+n) \log m)$ time. The lower bound reduces the search range of $k$ from $[1, k]$ to $[\hat{k}, k]$. This helps us further reduce the complexity of phase 1 in MOCHE from $O(m(n+m))$ to $O((n+m)\log(m) + (k - \hat{k})(n+m))$.

\begin{theorem}
\label{theorem:necessary_condition}
Given the KS test with a reference set $R$ and a test set $T$, for $h$ $(1 \leq h \leq m-1)$, there exists a qualified $h$-cumulative vector only if for each $i$ $(1 \leq i \leq q)$, the following holds.
\begin{subequations}
\label{eq:binary_search_eq_sys}
\begin{align}
    0 &\leq \lfloor \Gamma(i, h) + \Omega(h) \rfloor \label{eq:k_condition_1} \\
    \lceil M(i, h) - \Omega(h) \rceil &\leq h \label{eq:k_condition_3}\\
    M(i, h) - \Omega(h) &\leq  \Gamma(i, h) + \Omega(h) \label{eq:k_relax_inequality}
\end{align}
\end{subequations}
Moreover, if Equation~\ref{eq:binary_search_eq_sys} holds for $h>0$, then it also holds for $h+1$.

\begin{proof}
We first prove the necessary condition. 
Since there exists a qualified $h$-cumulative vector, by Theorem~\ref{theorem:l_u_bound_fast_check}, for each $i$ $(1 \leq i \leq q)$, $l_i^h \leq u_i^h$. $l_i^h$ and $u_i^h$ are the maximum and the minimum of the three terms in Equations~\ref{eq:unroll_bounds_of_c_i_lower} and~\ref{eq:unroll_bounds_of_c_i_upper}, respectively. Thus, every term in $u_i^h$ is larger than or equal to every term in $l_i^h$. Therefore, we immediately have Equations~\ref{eq:k_condition_1} and~\ref{eq:k_condition_3}, as well as the following. 
\begin{equation}
    \lceil M(i, h) - \Omega(h) \rceil \leq \lfloor \Gamma(i, h) + \Omega(h) \rfloor
    \label{eq:system_ineq_6}
\end{equation}
Since $M(i, h) - \Omega(h) \leq \lceil M(i, h) - \Omega(h) \rceil$ and $\lfloor \Gamma(i, h) + \Omega(h) \rfloor \leq \Gamma(i, h) + \Omega(h)$, Equation~\ref{eq:k_relax_inequality} follows Equation~\ref{eq:system_ineq_6} immediately. 

Next, we prove the monotonicity of the necessary condition with respect to $h$. For each inequality in Equation~\ref{eq:binary_search_eq_sys}, we show that for each $i$ $(1 \leq i \leq q)$, if a size $h$ $(1 \leq h \leq m - 2)$ satisfies the inequality, the size $h+1$ also satisfies the inequality.

\textbf{Equation~\ref{eq:k_condition_1}:}
Plugging the definitions of $\Gamma(i, h)$ and $\Omega(h)$ into Equation~\ref{eq:k_condition_1}, the inequality can be simplified to $\frac{\mathbf{C}_{T}[i]}{m - h} - \frac{\mathbf{C}_{R}[i]}{n} \geq -c_{\alpha}\sqrt{\frac{1}{m - h} + \frac{1}{n}}$.  Since $-c_{\alpha}\sqrt{\frac{1}{m - h} + \frac{1}{m}} > -c_{\alpha}\sqrt{\frac{1}{m - h - 1} + \frac{1}{m}}$ and $\frac{\mathbf{C}_{T}[i]}{m - h - 1} \geq \frac{\mathbf{C}_{T}[i]}{m - h}$, we have $\frac{\mathbf{C}_{T}[i]}{m - h - 1} - \frac{\mathbf{C}_{R}[i]}{n} > -c_{\alpha}\sqrt{\frac{1}{m - h - 1} + \frac{1}{n}}$, which can be simplified to $0 \leq \lfloor \Gamma(i, h + 1) + \Omega(h + 1) \rfloor$.

\textbf{Equation~\ref{eq:k_condition_3}:}
Plugging the definition of $M(i, h)$ into Equation~\ref{eq:k_condition_3}, we have $\lceil \Gamma(j, h) - \Omega(h) \rceil \leq h$, for each integer $j$ $(1 \leq j \leq i)$. 
Plugging the definitions of $\Gamma(j, h)$ and $\Omega(h)$ into the inequality, the inequality can be simplified to $\frac{\mathbf{C}_{T}[j]-h}{m - h} - \frac{\mathbf{C}_{R}[j]}{n} \leq c_{\alpha}\sqrt{\frac{1}{m - h} + \frac{1}{n}}$. Since $c_{\alpha}\sqrt{\frac{1}{m-h} + \frac{1}{n}} < c_{\alpha}\sqrt{\frac{1}{m-h-1} + \frac{1}{n}}$ and $\frac{\mathbf{C}_{T}[j] - h - 1}{m - h - 1} < \frac{\mathbf{C}_{T}[j] - h}{m - h}$, we have $\frac{\mathbf{C}_{T}[j] - h - 1}{m - h - 1} - \frac{\mathbf{C}_{R}[j]}{n} \leq c_{\alpha}\sqrt{\frac{1}{m - h - 1} + \frac{1}{n}}$, which can be simplified to $\Gamma(j, h + 1) - \Omega(h + 1)  \leq h$. Since $h$ is an integer, we immediately have $\lceil \Gamma(j, h + 1) - \Omega(h + 1) \rceil \leq h$. Applying the definition of $M(i, h)$, we have $\lceil M(i, h+1) - \Omega(h+1) \rceil \leq h$.

\textbf{Equation~\ref{eq:k_relax_inequality}:}
According to the definition of $M(i, h)$, from Equation~\ref{eq:k_relax_inequality}, we have $\Gamma(j, h) - \Omega(h) \leq  \Gamma(i, h) + \Omega(h)$, for each integer $j$ $(1 \leq j \leq i)$. Plugging the definitions of $\Omega(h)$ and $\Gamma(j, h)$ into the inequality, the inequality can be simplified to $-2c_{\alpha}\sqrt{\frac{1}{m - h} + \frac{1}{n}} \leq \frac{\mathbf{C}_{T}[i] - \mathbf{C}_{T}[j]}{m - h} - \frac{1}{n}(\mathbf{C}_{R}[i] - \mathbf{C}_{R}[j])$. Since $\frac{\mathbf{C}_{T}[i] - \mathbf{C}_{T}[j]}{m - h - 1} > \frac{\mathbf{C}_{T}[i] - \mathbf{C}_{T}[j]}{m - h}$ and $-2c_{\alpha}\sqrt{\frac{1}{m - h} + \frac{1}{n}}   > -2c_{\alpha}\sqrt{\frac{1}{m - h-1} + \frac{1}{n}}$,  we have $-2c_{\alpha}\sqrt{\frac{1}{m - h - 1} + \frac{1}{n}} < \frac{\mathbf{C}_{T}[i] - \mathbf{C}_{T}[j]}{m-h-1} - \frac{1}{n}(\mathbf{C}_{R}[i] - \mathbf{C}_{R}[j])$, which can be simplified to $\Gamma(j, h + 1) - \Omega(h + 1) \leq  \Gamma(i, h + 1) + \Omega(h + 1)$.  Applying the definition of $M(i, h)$, we have $M(i, h + 1) - \Omega(h + 1) \leq  \Gamma(i, h + 1) + \Omega(h + 1)$.
\end{proof}
\end{theorem}

The smallest integer $\hat{k}$ that satisfies the necessary condition in Theorem~\ref{theorem:necessary_condition} is a lower bound on the size $k$ of the explanations. We do not need to check any $h$-subset smaller than $\hat{k}$, as they are guaranteed not to contain a qualified $h$-cumulative vector. 
Based on the monotonicity of Equation~\ref{eq:binary_search_eq_sys} with respect to $h$, we can apply binary search to find the smallest integer $\hat{k}$ that satisfies Theorem~\ref{theorem:necessary_condition}. 
For $h \in [1, m-1]$, it takes $O(n+m)$ time to verify the $q$ groups of inequalities in Theorem~\ref{theorem:necessary_condition}, because $q \leq n + m$. Therefore, the overall time complexity of finding $\hat{k}$ is $O((m + n) \log m)$.  Once $\hat{k}$ is found, we iteratively use Theorem~\ref{theorem:l_u_bound_fast_check} to find the exact size of explanations.  The overall time complexity of finding the exact size of explanation is $O((m+n)\log m + (m+n)(k-\hat{k}))$, where $k$ is the exact size.  In the worst case, $k-\hat{k}=O(m)$, and the complexity is still $O(m(m+n))$.  However, as verified by our experiments, $k - \hat{k}$ is often a very small number and our technique can significantly improve the efficiency of searching the size of explanations.

\begin{example}\rm
\label{ex:binary_search}
\mc{Consider the failed KS test in Example~\ref{ex:bounds}. We apply binary search to find the lower bound $\hat{k} \in [1, 3]$. We start with $h=\lfloor(1+3)/2\rfloor=2$ and find that $h=2$ satisfies Theorem~\ref{theorem:necessary_condition}. Thus, $\hat{k}\leq 2$. Then, we search the left half $[1, 2]$ and set $h=\lfloor(1+2)/2\rfloor=1$. As $\lceil M(1, h) - \Omega(h) \rceil = 2$, Equation~\ref{eq:k_condition_3} does not hold and thus $h=1$ does not satisfy Theorem~\ref{theorem:necessary_condition}. 
This concludes that $\hat{k}=2$.}
\end{example}
\section{Generating Most Comprehensible Explanations}
\label{sec:generate_interpretation}

Given the size of explanations $k$, the brute force method takes $O({m \choose k}(m+n-k)\log(m+n-k))$ time to find the most comprehensible explanation by enumerating the $k$-subsets of $T$. In this section, we develop a method to directly construct the most comprehensible explanation in $O(m(n+m))$ time without enumerating the $k$-subsets.

\label{sec:construct_most_comp_interpretation}

An $h$-subset $S \subset T$ is called an \emph{$h$-partial explanation} if there exists an explanation that is a superset of $S$. When it is clear from the context, we also call $S$ a \emph{partial explanation} for short. 

According to Definition~\ref{def:interpretation}, the most comprehensible explanation is the explanation that has the smallest lexicographical order. This property facilitates the design of our construction algorithm. Our algorithm scans the data points in $T$ in the order of $L$ and selects the first data point $x_{i_1}$ that is in an explanation, that is, $x_{i_1}$ is a $1$-partial explanation.  Since $x_{i_1}$ is the first such data point in $L$, the most comprehensible explanation must contain $x_{i_1}$, otherwise we have the contradiction that the explanation containing $x_{i_1}$ precedes the most comprehensible explanation in the lexicographical order.  Then, the algorithm continues to scan the points after $x_{i_1}$ in $L$, still in the order of $L$, and finds the next data point $x_{i_2}$ such that $\{x_{i_1}, x_{i_2}\}$ are part of an explanation, that is, $\{x_{i_1}, x_{i_2}\}$ is a $2$-partial explanation.  Clearly, $\{x_{i_1}, x_{i_2}\}$ is part of the most comprehensive explanation.  The search continues until $k$ points are obtained, which is the most comprehensible explanation.  The construction method is summarized in Algorithm~\ref{algo:interpretation_generation}. 

\begin{algorithm}[t]
\caption{Find the most comprehensible explanation}
\label{algo:interpretation_generation}
\small
\KwIn{a reference set $R$, a test set $T$, a significance level $\alpha$, a preference list $L$, the size of explanations $k$}
\KwOut{$I:=$ the most comprehensible explanation}
\BlankLine

Initialize $I \leftarrow \emptyset$

$T \leftarrow $ sort the data points in $T$ in the order of $L$

\For{$i \leftarrow 1; i \leq |T|; i++$}{
    
    \lIf{$I \cup \{T[i]\}$ \text{is a partial explanation}}{
        $I \leftarrow I \cup \{T[i]\}$
    }
    
    \lIf{$|I| = k$}{
        \Return $I$
    }
}
\end{algorithm}

\label{sec:construct_check_partial_interpretation}

Now, the remaining question is how we can determine whether an $h$-subset $S$ is a partial explanation. We first establish that a subset $S$ is a partial explanation if and only if there exists a qualified $k$-cumulative vector, which satisfies a small group of inequalities derived from $S$. Then, we introduce a sufficient and necessary condition for the existence of such a $k$-cumulative vector, which can be efficiently checked in $O(n+m)$ time.

\begin{lemma}
\label{lemma:subset_interpretation_existence}
Given the KS test with a reference set $R$ and a test set $T$, for a subset $S \subset T$, $S$ is a partial explanation if and only if there exists a qualified $k$-cumulative vector $\mathbf{C}$, such that the following inequality holds for $1 \leq i \leq q$,
\begin{equation}
\label{eq:configuration_satisfy_s}
    \mathbf{C}[i] - \mathbf{C}[i-1] \geq \mathbf{C}_{S}[i] - \mathbf{C}_{S}[i-1]
\end{equation}
\begin{proof}
(Necessity) Since $S$ is a partial explanation, by definition, there exists an explanation $\mathcal{I}$ such that $S \subseteq \mathcal{I}$. Denote by $\mathbf{C}$ the qualified $k$-cumulative vector of $\mathcal{I}$. For each $i$ $(1 \leq i \leq q)$, since $\mathbf{C}[i] - \mathbf{C}[i-1]$ and $\mathbf{C}_{S}[i] - \mathbf{C}_{S}[i-1]$ are the numbers of times $x_i$ appearing in $\mathcal{I}$ and $S$, respectively, and $S \subseteq \mathcal{I}$, $\mathbf{C}[i] - \mathbf{C}[i-1] \geq \mathbf{C}_{S}[i] - \mathbf{C}_{S}[i-1]$ holds. The necessity follows.

(Sufficiency) Assume a qualified $k$-cumulative vector $\mathbf{C}$ that satisfies Equation~\ref{eq:configuration_satisfy_s}. Let $\mathcal{I}$ be the explanation corresponding to $\mathbf{C}$. For each data point $x_i \in S$, $x_i$ appears $\mathbf{C}_{S}[i] - \mathbf{C}_{S}[i-1]$ times in $S$.  Due to Equation~\ref{eq:configuration_satisfy_s}, $x_i$ appears in $\mathcal{I}$ the same or more number of times. Thus, $S \subseteq \mathcal{I}$ and $S$ is a partial explanation.
\end{proof}
\end{lemma}


Next,  we derive a sufficient and necessary condition for the existence of such a $k$-cumulative vector $\mathbf{C}$ by investigating the lower bound and the upper bound of each element $\mathbf{C}[i]$. Since $\mathbf{C}$ is a qualified $k$-cumulative vector, by Theorem~\ref{theorem:l_u_bound_fast_check}, $l_i^k \leq \mathbf{C}[i] \leq u^k_i$. 
Equation~\ref{eq:configuration_satisfy_s} can be rewritten as $\mathbf{C}[i-1] \leq \mathbf{C}[i] - \mathbf{C}_{S}[i] + \mathbf{C}_{S}[i-1]$. That is, the upper bound of $\mathbf{C}[i-1]$ dependents on the upper bound of $\mathbf{C}[i]$. Denote by $\bar{l}_i=l_i^k$ a lower bound of $\mathbf{C}[i]$ and by $\bar{u}_i$ an upper bound of $\mathbf{C}[i]$.  Since $\mathbf{C}[i] \leq \bar{u}_{i}$ and $\mathbf{C}[i-1] \leq u_{i-1}^k$, about the upper bounds we have, for $i$ $(1 \leq i \leq q)$,
\begin{equation}\label{eq:u_bar}
    \bar{u}_{i-1} = \min(u^k_{i-1}, \bar{u}_{i} - \mathbf{C}_{S}[i] + \mathbf{C}_{S}[i-1]).
\end{equation}

Given the size of explanations $k$, we first compute $u_i^k$ for each $i$ $(1 \leq i \leq q)$ by Equation~\ref{eq:unroll_bounds_of_c_i_upper}. Then, we iteratively compute $\bar{u}_i$ for each $i$ $(0 \leq i \leq q)$. We define $\bar{u}_q = u^k_q$ and plug $\bar{u}_q$ into Equation~\ref{eq:u_bar} to compute $\bar{u}_{q-1}$, and iteratively compute the upper bound of each $\mathbf{C}[i]$ of a qualified $k$-cumulative vector $\mathbf{C}$ that satisfies Equation~\ref{eq:configuration_satisfy_s}. 

Since $\bar{u}_i$ depends on $u_i^k$, once the size of explanations $k$ is determined using the techniques developed in Section~\ref{sec:interpretation_size}, we can compute the value $\bar{u}_i$. \mc{Based on a similar intuition as Theorem~\ref{theorem:l_u_bound_fast_check},} we can use the lower bound $\bar{l}_i$ and the upper bound $\bar{u}_i$ to derive a sufficient and necessary condition for the existence of a qualified $k$-cumulative vector $\mathbf{C}$ that satisfies Equation~\ref{eq:configuration_satisfy_s} as stated in the following result,
and thus decide whether an $h$-subset $S$ is a partial explanation.

\begin{theorem}
\label{theorem:correctness_algorithm_interpretation_superset_check}
Given the KS test with a reference set $R$ and a test set $T$, for a subset $S \subset T$, there exists a qualified $k$-cumulative vector $\mathbf{C}$ that satisfies Equation~\ref{eq:configuration_satisfy_s} if and only if for each $i$ $(0 \leq i \leq q)$, $\bar{l}_i \leq \bar{u}_i$.
\begin{proof}
(Sufficiency) Given $S$, assume for each $i$ $(0 \leq i \leq q)$, $\bar{l}_i \leq \bar{u}_i$. We construct a $k$-cumulative vector $\mathbf{C}$ such that for each $i$ $(0 \leq i \leq q)$, $\mathbf{C}[i] = \bar{u}_i$. We show that $\mathbf{C}$ is a qualified $k$-cumulative vector and also satisfies Equation~\ref{eq:configuration_satisfy_s}.

We first prove that $\mathbf{C}$ is a qualified $k$-cumulative vector by showing that $\mathbf{C}[0] = 0$, and each $\mathbf{C}[i]$ $(1 \leq i \leq q)$ satisfies Equations~\ref{eq:lower_range_of_c_i} and~\ref{eq:upper_range_of_c_i}. Since $l^k_0 = \bar{l}_0 \leq \bar{u}_0 \leq u^k_0 = 0$, we have $\bar{l}_0 = \bar{u}_0 = 0$ and thus $\mathbf{C}[0] = 0$.
Plugging $\mathbf{C}[i-1] = \bar{u}_{i-1}$ and $\mathbf{C}[i] = \bar{u}_{i}$ into Equation~\ref{eq:u_bar}, we have $\mathbf{C}[i-1] \leq \mathbf{C}[i]$. Since $\bar{l}_i = l^k_i \leq \mathbf{C}[i]$, from Equation~\ref{eq:unroll_bounds_of_c_i_lower}, we have $\lceil \Gamma(i, h) - \Omega(h) \rceil \leq \mathbf{C}[i]$ and $h-m+\mathbf{C}_{T}[i] \leq \mathbf{C}[i]$. Therefore, $\mathbf{C}[i]$ satisfies Equation~\ref{eq:lower_range_of_c_i}.

Plugging $\mathbf{C}[i-1] = \bar{u}_{i-1}$ into Equation~\ref{eq:u_bar}, the value of $\mathbf{C}[i-1]$  falls into one of the following two cases. 

Case 1: $\mathbf{C}[i-1] = u^k_{i-1}$. As $u_i^k$ is derived by plugging $\mathbf{C}[i-1] = u_{i-1}^k$ into Equation~\ref{eq:upper_range_of_c_i}, we have $u^k_i \leq \mathbf{C}_{T}[i] - \mathbf{C}_{T}[i-1] + u^k_{i-1}$. Since $\mathbf{C}[i] \leq u_i^k$ and $\mathbf{C}[i-1] = u^k_{i-1}$, we have $\mathbf{C}[i] \leq  \mathbf{C}_{T}[i] - \mathbf{C}_{T}[i-1] + \mathbf{C}[i-1]$.

Case 2: $\mathbf{C}[i-1] = \bar{u}_{i} - \mathbf{C}_{S}[i] + \mathbf{C}_{S}[i-1]$. Since $S \subset T$,  $\mathbf{C}_{S}[i] - \mathbf{C}_{S}[i-1] \leq \mathbf{C}_{T}[i] - \mathbf{C}_{T}[i-1]$. Since $\mathbf{C}[i] = \bar{u}_i$, we have $\mathbf{C}[i] \leq \mathbf{C}_{T}[i] - \mathbf{C}_{T}[i-1] + \mathbf{C}[i-1]$.

Since $\mathbf{C}[i] = \bar{u}_i \leq u^k_i$, from Equation~\ref{eq:unroll_bounds_of_c_i_upper}, we have $\mathbf{C}[i] \leq k$ and $\mathbf{C}[i] \leq \lfloor \Gamma(i, k) + \Omega(k) \rfloor$. Therefore, $\mathbf{C}[i]$ satisfies Equation~\ref{eq:upper_range_of_c_i}. By Lemma~\ref{theorem:critical_set_c_i_range}, $\mathbf{C}$ is a qualified $k$-cumulative vector.

Since for each $i$ $(0 \leq i \leq q)$, $\mathbf{C}[i]=\bar{u}_i$, from Equation~\ref{eq:u_bar}, we can derive that every $\mathbf{C}[i]$ satisfies Equation~\ref{eq:configuration_satisfy_s}. 

(Necessity) Given $S$, assume a qualified $k$-cumulative vector $\mathbf{C}$ that satisfies Equation~\ref{eq:configuration_satisfy_s}. Since $\bar{l}_i$ and $\bar{u}_i$ are the lower and the upper bound of $\mathbf{C}[i]$, respectively, the necessity follows immediately.
\end{proof}
\end{theorem}

\begin{example}\rm
\label{ex:interpretation_construction}
\mc{Consider the failed KS test in Example~\ref{ex:bounds}, where the size of explanations $k$ is 2. Suppose a user provides a preference list $L=[t_4, t_3, t_2, t_1]$. 
We initialize the constructed explanation $I=\emptyset$ and scan the data points in $T$ in the order of $L$. For the first scanned data point $t_4=20$, we check if $S=I \cup \{t_4\}$ is a partial explanation. 
By Equation~\ref{eq:u_bar}, the upper bound $\bar{u}_3=1$.
As the lower bound $\bar{l}_3 = l_3^k = 2 > \bar{u}_3$, by Theorem~\ref{theorem:correctness_algorithm_interpretation_superset_check}, $S$ is not a partial explanation and thus $t_4$ is not in any explanations. }

\mc{We repeat the same step for the second scanned data point $t_3=12$. When $S=\{t_3\}$, $(\bar{l}_0, \bar{u}_0)=(0, 0)$, $(\bar{l}_1, \bar{u}_1)= (0, 1)$, and $(\bar{l}_2, \bar{u}_2) = (\bar{l}_3, \bar{u}_3) = (\bar{l}_4, \bar{u}_4)=(2, 2)$.
By Theorem~\ref{theorem:correctness_algorithm_interpretation_superset_check}, $S$ is a partial explanation and thus we add $t_3$ to $I$. The third scanned data point $t_2=13$ is added to $I$ for the same reason. As the size of $I=\{t_3, t_2\}$ is equal to $k$, $I$ is the most comprehensible explanation on the failed KS test.}
\end{example}

Given an explanation size $k$ and a subset $S$, it takes $O(m+n)$ time to verify the $q+1$ groups of inequalities in Theorem~\ref{theorem:correctness_algorithm_interpretation_superset_check}, because $q \leq n + m$. Since for each data point $t_i$ in $T$, we need to check whether $I \cup \{t_i\}$ is a partial explanation, where $I$ is the partial explanation found so far, the overall time complexity of constructing the most comprehensible explanation is $O(m(n+m))$. 

As shown in Section~\ref{sec:interpretation_size}, it takes $O((m+n)\log(m)) + O((n+m)(k - \hat{k}))=O((m+n)(\log m+k-\hat{k}))$ time to identify the explanation size. In total, our method takes $O(m(n+m))$ time to find the most comprehensible explanation for a failed KS test. 
\section{Experiments}
\label{sec:exp}

In this section, we evaluate the effectiveness of most comprehensible counterfactual explanations, and the efficiency and scalability of MOCHE. We describe the datasets and the 
experiment settings
in Section~\ref{sec:settings}.  Counterfactual explanations on a failed KS test have two fundamental requirements, being small and reversing the failed KS test. 
In Section~\ref{subsec:exp_small}, we evaluate the size of our explanations. In Section~\ref{sec:exp_effective}, we evaluate whether our explanations can reverse failed KS tests.
In Section~\ref{subsec:exp_effective}, we investigate the effectiveness of our method.
Last, in Section~\ref{subsec:exp_scale}, we verify the efficiency and scalability of our proposed method. 

\subsection{Datasets and Experiment Settings}
\label{sec:settings}

\subsubsection{Dataset Construction}

\mc{We conduct experiments using a COVID-19 dataset and 6 univariate time series datasets in the Numenta Anomaly Benchmark (NAB) repository~\cite{lavin2015evaluating}.}

\paragraph{\mc{COVID-19 Data}}

\mc{The COVID-19 dataset\footnote{http://www.bccdc.ca/health-info/diseases-conditions/covid-19/data} is described in Examples~\ref{ex:motivation} and~\ref{ex:lexicographic_order}.   
The 10 age groups in the dataset are encoded from young to old by integers from 1 to 10. 
We use the cases reported in August and September 2020 to build the reference set and the test set, respectively.
The KS test fails at significance level $0.05$, which indicates that the infected cases in those two months unlikely follow the same distribution on age groups.  In Section~\ref{subsec:exp_effective}, as a case study we interpret the failed KS test to find the data points that 
may likely be relevant to the failure.}

\mc{We use the population descending order of the HAs to generate the preference list $L$ of data points in the test set. The data points from the same HAs are sorted arbitrarily. We obtain the populations of the HAs from the website of Statistics Canada\footnote{\url{https://www12.statcan.gc.ca/census-recensement/2016/dp-pd/prof/index.cfm}}.}

\paragraph{Time Series Data}
Each dataset in the NAB repository contains $6$ to $10$ time series and each time series contains 1,000 to 20,000 observations. For each time series, the ground truth labels of abnormal observations are available. The AWS server metrics (AWS) dataset contains the time series of the CPU Utilization, Network Bytes In, and Disk Read Bytes of an AWS server. The online advertisement clicks (AD) dataset contains the time series of online advertisement clicking rates and cost per thousand impressions. The freeway traffic (TRF) dataset contains the time series of occupancy, speed, and travel time of freeway traffics collected by specific sensors. The Tweets (TWT) dataset contains the time series of numbers of Twitter mentions of publicly-traded companies such as Google, IBM, and Apple. The miscellaneous known causes (KC) dataset contains the time series from multiple domains, including machine temperature, number of NYC taxi passengers, and CPU usage of an AWS server. The artificial (ART) dataset contains the artificially-generated time series with varying types of distribution drifts~\cite{kifer2004detecting}. Table~\ref{tbl:datasets} shows some statistics of the datasets.

\begin{table}[t]
\scalebox{0.80}{
\begin{tabular}{|l|c|c|}
\hline\small
Dataset    & \# Time series & Length               \\ \hline
AWS   & 17   & 1,243 $\sim$ 4,700       \\ \hline
AD & 6    & 1,538 $\sim$ 1,624       \\ \hline
TRF    & 7    & 1,127 $\sim$ 2,500     \\ \hline
TWT      & 10   & 15,831 $\sim$ 15,902  \\ \hline
KC & 7    & 1,882 $\sim$ 22,695    \\ \hline
ART & 6    & 4032               \\ \hline
\end{tabular}
}
\caption{Some statistics of the datasets.}\vspace{-5mm}
\label{tbl:datasets}
\end{table}

We run a sliding window $W$ of size $w$ to obtain the reference set, and use the window of the same size following $W$ immediately without any overlap as the test set.  
The reference set and the test set are muti-sets consisting of the observation values in corresponding sliding windows.
The KS test is conducted multiple times as the sliding windows run through a time series. A failed KS test indicates that the time series has a distribution drift~\cite{kifer2004detecting}. We interpret the failed KS test to find the data points that are likely relevant to the failure. 
 
The significance level of the KS test is always set to $0.05$ following the convention in statistical testing. We use a variety of window sizes, including 100, 200, 300, 1,000, 1,500, and 2,000.

We apply a widely used time series outlier detection method, Spectral Residual~\cite{ren2019time} to automatically generate the preference lists $L$ of data points in the test sets. This preference list $L$ reflects a user's domain knowledge about data abnormality. Data points with larger outlying scores are ranked higher in $L$. The data points with the same outlying scores are sorted randomly. We use the published Python codes of Spectral Residual\footnote{https://github.com/SeldonIO/alibi-detect} with the default parameters.

\subsubsection{Baselines} 

To the best of our knowledge, interpreting failed KS tests has not been studied in literature. To evaluate the performance of MOCHE (M for short in figures), we design six baselines.



\textbf{Greedy} (GRD for short) generates a counterfactual explanation $\mathcal{I}$ by greedily selecting the first $l$ data points in $L$ such that $R$ and $T \setminus \mathcal{I}$ can pass the KS test. When the preference list is generated by an outlier detection method, Greedy can be regarded as an extension of the outlier detection method to interpret failed KS tests.

\textbf{Extended-CornerSearch} (CS for short) is extended from CornerSearch~\cite{croce2019sparse}, a state-of-the-art $L_{0}$-norm adversarial attack method on image classifiers. CornerSearch generates adversarial images by randomly searching a small portion of the top-$K$ important pixels of input images and masking them to $0$ or $1$. Although CornerSearch is not proposed to interpret failed KS tests, it may be extended to serve the purpose. The Extended-CornerSearch treats data points as pixels and perturbs the selected data points $\mathcal{I}$ by removing them from $T$. 
After applying each perturbations, it conducts the KS test on $R$ and $T \setminus \mathcal{I}$ to check if $\mathcal{I}$ is an explanation, meaning passing the KS test so that the unchanged part is regarded as normal by classifiers.

\textbf{Extended-GRACE} (GRC for short) is a direct extension from GRACE~\cite{grace}, the state-of-the-art counterfactual explanation method on neural networks. To interpret a prediction on an input vector $\mathbf{x}$, GRACE perturbs the most important $K$ features of $\mathbf{x}$, which are ranked by an external method, to change the prediction. GRACE only accepts vectors as inputs and generates explanations by minimizing a target classifier's prediction scores. Correspondingly, we extend GRACE to interpret failed KS tests by first accommodating the inconsistency between the inputs of GRACE and our problem through a mapping from an $m$-dimensional vector $\mathbf{x}$ to a subset $S \subseteq T$, where $m=|T|$. We project $\mathbf{x}$ to its nearest 0-1 vector and put the $i$-th data point $t_i$ into $S$ if the $i$-th element of the vector is $0$. Next, we extend the objective function of GRACE to find explanations on failed KS tests by perturbing $\mathbf{x}$ to minimize 
$g(\mathbf{x})=\sqrt{\frac{n*(m-|S|)}{n+(m-|S|)}} D(R, T \setminus S)$,
where $S$ is the set of data points picked by vector $\mathbf{x}$. Based on the definition of the KS test, $S$ is an explanation on the failed KS test if $g(\mathbf{x})$ is smaller than the critical value $c_{\alpha}$. Since $g(\mathbf{x})$ is not differentiable, we adopt the zeroth order optimization algorithm in~\cite{cheng2018query} to solve the problem. We skip the entropy-based feature selection step used in GRACE, as it requires access to training data of classifiers, which is not available in our problem setting. 

\mc{\textbf{Extended-D3} is extended from D3~\cite{10.5555/1182635.1164145}, an outlier detection method on data streams. Given a set of historical data points $X$, a new coming data point $x$ is detected as an outlier if $x$ has a low probability density in $X$. As D3 is not designed for interpreting failed KS tests, we extend D3 to serve the purpose. The Extended-D3 selects the data points in $T$ that have high probability densities in $T$ and low probability densities in $R$. Specifically, denote by $f_{R}$ and $f_{T}$ the estimated probability density functions of $R$ and $T$, respectively. Extended-D3 sorts the data points $t_i$ in $T$ in $\frac{f_{T}(t_i)}{f_{R}(t_i)}$ descending order. Then, it greedily selects the first $l$ data points such that $R$ and $T \setminus \mathcal{I}$ can pass the KS test. By default, $f_{R}$ and $f_{T}$ are learned using the same way as D3. For the COVID-19 dataset, as the data values are discrete, we use the empirical probability mass functions of $R$ and $T$ as $f_{R}$ and $f_{T}$, respectively. As Extended-D3 cannot take user preferences as input, it cannot produce comprehensible explanations. When the context is clear, we call this baseline method \textbf{D3} for short.}

\mc{\textbf{Extended-STOMP} (STMP for short) is extended from STOMP~\cite{yeh2016matrix}, a widely used anomalous subsequence detection method on time series. Given a regular time series  $\mathbf{N}$, a query time series $\mathbf{Q}$, and a subsequence length $q$, STOMP aims to detect anomalous subsequences of length $q$ (each is called a $q$-subsequence) in $\mathbf{Q}$. STOMP applies z-normalization on each subsequence and detects subsequences with anomalous shapes~\cite{yeh2016matrix}. }

\mc{We extend STOMP to interpret the failed KS tests conducted on the time series datasets. For a failed KS test, let $\mathbf{N}$ and $\mathbf{Q}$ be the corresponding time series segments of the reference set and the test set, respectively. Extended-STOMP sorts the $q$-subsequences of $\mathbf{Q}$ by their anomalous scores in decreasing order. Then, the algorithm greedily selects the data points from the first $l$ subsequences such that $R$ and $T \setminus \mathcal{I}$ can pass the KS test. Same as D3, Extended-STOMP cannot produce comprehensible explanations. }

\mc{\textbf{Extended-Series2Graph} (S2G for short) is extended from Series2Graph~\cite{boniol2020series2graph}, a state-of-the-art anomalous subsequence detection method on time series. Series2Graph takes the same input as STOMP. 
It detects  $q$-subsequences of $\mathbf{Q}$ with anomalous shapes by learning a subsequence embedding model. We extend Series2Graph to interpret failed KS tests in the same way as Extended-STOMP. Same as D3, S2G cannot produce comprehensible explanations. }

\subsubsection{Parameter Settings} 
By default the significance level in all KS tests is fixed to $0.05$.

We adopt the same parameter setting used in~\cite{croce2019sparse} for CS. For GRC, we set $K=100$ to be consistent with mc{CS} and set the remaining parameters to the same as~\cite{grace}. We use the same parameters as~\cite{cheng2018query} for the zeroth optimization algorithm used in GRC. 
\mc{The parameters of D3 are set to the same as~\cite{10.5555/1182635.1164145}.
We test STMP and S2G with a variety  of $q$ values, including $5\%|T|$, $10\%|T|$, $20\%|T|$,  and $40\%|T|$. Since $q=5\%|T|$ outperforms the other settings on producing small explanations, we choose $q=5\%|T|$ for STMP and S2G in all experiments. The remaining parameters of STMP and S2G are set to the same as~\cite{yeh2016matrix} and ~\cite{boniol2020series2graph}, respectively.}

We use the published Python codes of STOMP\footnote{\url{https://matrixprofile.org}} and Series2Graph\footnote{\url{http://helios.mi.parisdescartes.fr/~themisp/series2graph/}}. The remaining algorithms are implemented in Python. All experiments are conducted on a server with two Xeon(R) Silver 4114 CPUs (2.20GHz), four Tesla P40 GPUs, 400GB main memory, and a 1.6TB SSD running Centos 7 OS. Our source code is published on GitHub \url{https://github.com/research0610/MOCHE}.

Since CS and GRC cannot return explanations for all failed KS tests in 24 hours, for each combination of time series and window size, we uniformly sample 10 failed KS tests, where the test sets contain the corresponding ground truth of abnormal observations. We conduct all experiments on the sampled $2,690$ failed KS tests. 

\subsection{Conciseness}
\label{subsec:exp_small}

Small explanations help users focus on predominant factors in a decision~\cite{wang2019designing}. Therefore, being small is a key preference on counterfactual explanations~\cite{moraffah2020causal, grace}. 

\mc{
We design a binary variable \emph{Is-Smallest-Explanation (ISE)} in the performance study of the compared methods in producing small counterfactual explanations. For the explanations produced by all methods on the same failed KS test, the ISE of the smallest explanation is $1$, and $0$ for the other explanations.
}

We evaluate MOCHE and the six baseline methods in ISE on the failed KS tests of the time series datasets. GRC and CS cannot find counterfactual explanations for some failed KS tests. To fairly compare the methods, among the 2,690 failed KS tests in those datasets, in this experiment we only consider the 847 ones (31.4\%) where all methods can generate counterfactual explanations. Figure~\ref{fig:concise_sr_ais} shows the average ISE of all explanations.

\begin{figure}[t]
    \centering
    \includegraphics[width=0.75\linewidth]{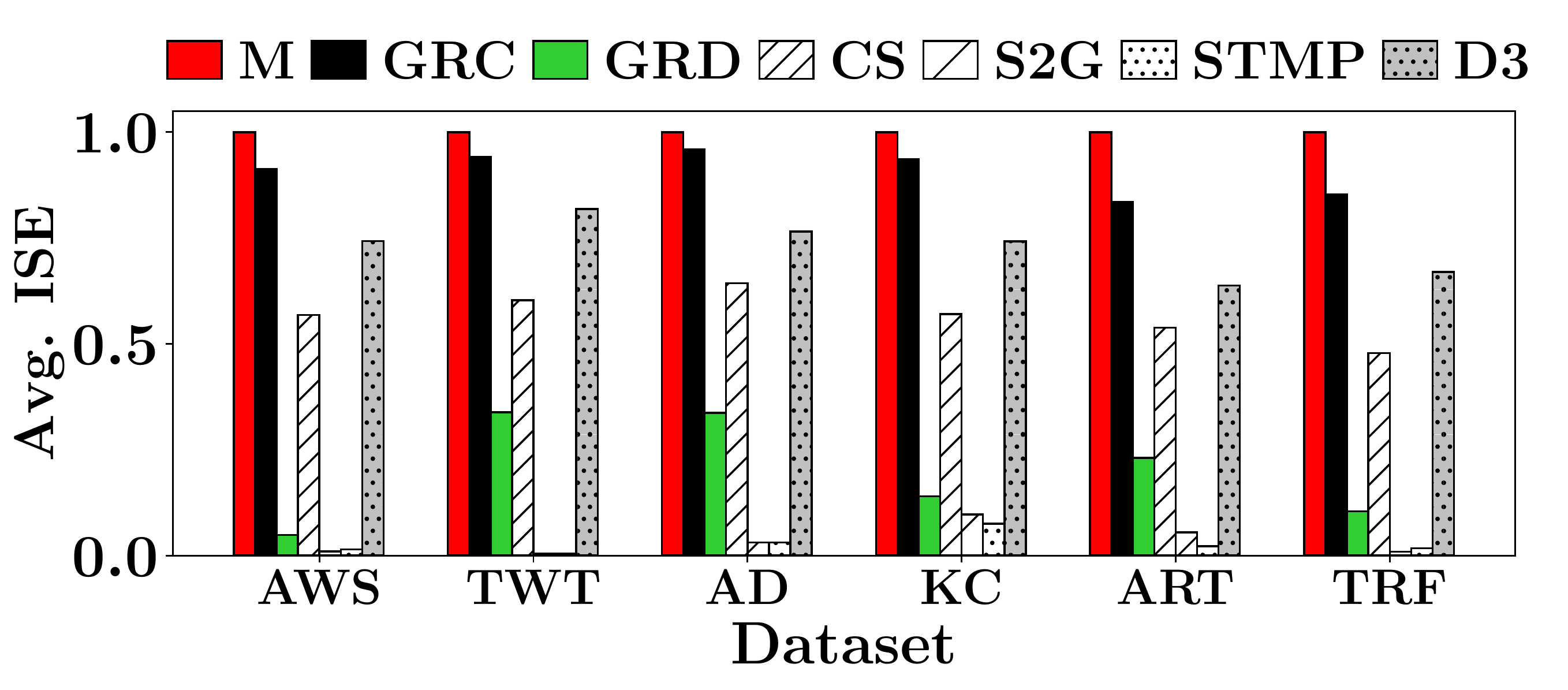}
    \caption{\mc{The average ISE, the larger the better.}}
    \label{fig:concise_sr_ais}
\end{figure}

\mc{STMP and S2G perform poorly. They choose some data points from the outlying subsequences as explanations on a failed KS test. Their outlying scores are computed on normalized subsequences, whose original distributions are changed~\cite{boniol2020series2graph}. Therefore, the data points from the outlying subsequences cannot explain why the KS test detects the distribution change between the reference set and the test set, and thus cannot find the smallest explanations on most of the failed KS tests.}

\mc{D3 outperforms STMP and S2G. 
It interprets by comparing the estimated distributions of the reference set and the test set. 
However, limited by the approximation quality of its distribution estimator, D3 cannot always produce the smallest explanations.}

GRD and CS do not perform well. Both methods generate explanations by taking the first several data points in the preference lists until the picked data points reverse the KS tests. However, since the preference lists are generated by a method independent from the KS test, some data points that are not highly relevant to the failure of the KS test may still be ranked high in the preference lists.
As a result, those two methods may select many data points irrelevant to the failure of the KS test and lead to unnecessarily large explanations.

As a counterfactual explanation method, GRC generates explanations by solving an optimization problem, which allows it to re-rank the data points based on their effects on the KS tests. Therefore, as shown in Figure~\ref{fig:concise_sr_ais}, GRC finds smaller explanations than the other baseline methods. However, GRC still cannot guarantee to find the smallest explanations all the time, because its objective function is non-differentiable and hard to minimize.

MOCHE guarantees to find the smallest explanation and thus has ISE value $1$ in all cases. 


\subsubsection{Contrastivity}
\label{sec:exp_effective}

A counterfactual explanation on a failed KS test should reverse the failed KS test into a passed one. In this subsection, we quantitatively evaluate the performance of the methods in providing explanations that can reverse failed KS tests.

To measure the capability of a method, we use the \emph{reverse factor (RF)}, which is the ratio $   RF=\frac{\text{Number of reversed failed KS tests}}{\text{Total number of failed KS tests}}$.
The larger the RF value, the stronger capability a method reversing failed KS tests.

Since GRC and CS cannot produce all results within 24 hours on some data sets, we constrain the two methods to only generate explanations using the top-$100$ ranked data points in the preference lists $L$. In other words, GRC and CS abort if a failed KS test does not have a counterfactual explanation that is a subset of the top-$100$ data points. To compare the methods in a fair manner, in this experiment, we only count the 1,293 ($48.1\%$) among the 2,690 failed KS tests where GRC and CS do not abort. \mc{Table~\ref{tbl:reverse_factor} shows the RF of CS and GRC. The RF values of the other methods are always 1 on all datasets.}

CS and GRC cannot find counterfactual explanations for a large number of failed KS tests. The non-differential objective function of GRC is hard to optimize.
CS likely samples the top-ranked data points in a preference list~\cite{croce2019sparse}. 
If the top-ranked data points are not relevant to the failure of a KS test, CS cannot reverse the failed KS test within its optimization steps.
One may improve the RF of GRC and CS by more optimization steps. However, as to be shown in Section~\ref{subsec:exp_scale}, these two methods are very slow, and more optimization steps make them even slower.
The other baselines have a good RF.  However, as shown in Figure~\ref{fig:concise_sr_ais}, those methods tend to find large subsets of the test set as explanations, which are not informative~\cite{grace, moraffah2020causal}

\mc{The RF of MOCHE is $1$ on all datasets. MOCHE guarantees to produce the most comprehensible counterfactual explanations.}

\subsection{Effectiveness and Case Study}
\label{subsec:exp_effective}

\mc{A counterfactual explanation on a failed KS test is effective if removing the explanation from the test set could make the distributions of the reference set and the test set similar. In this subsection, we first quantitatively evaluate the effectiveness of the explanations generated by all methods.
Then, we conduct a case study to illustrate the effectiveness of the most comprehensible explanations.}

\begin{figure}[t]
    \centering
    \includegraphics[width=0.75\linewidth]{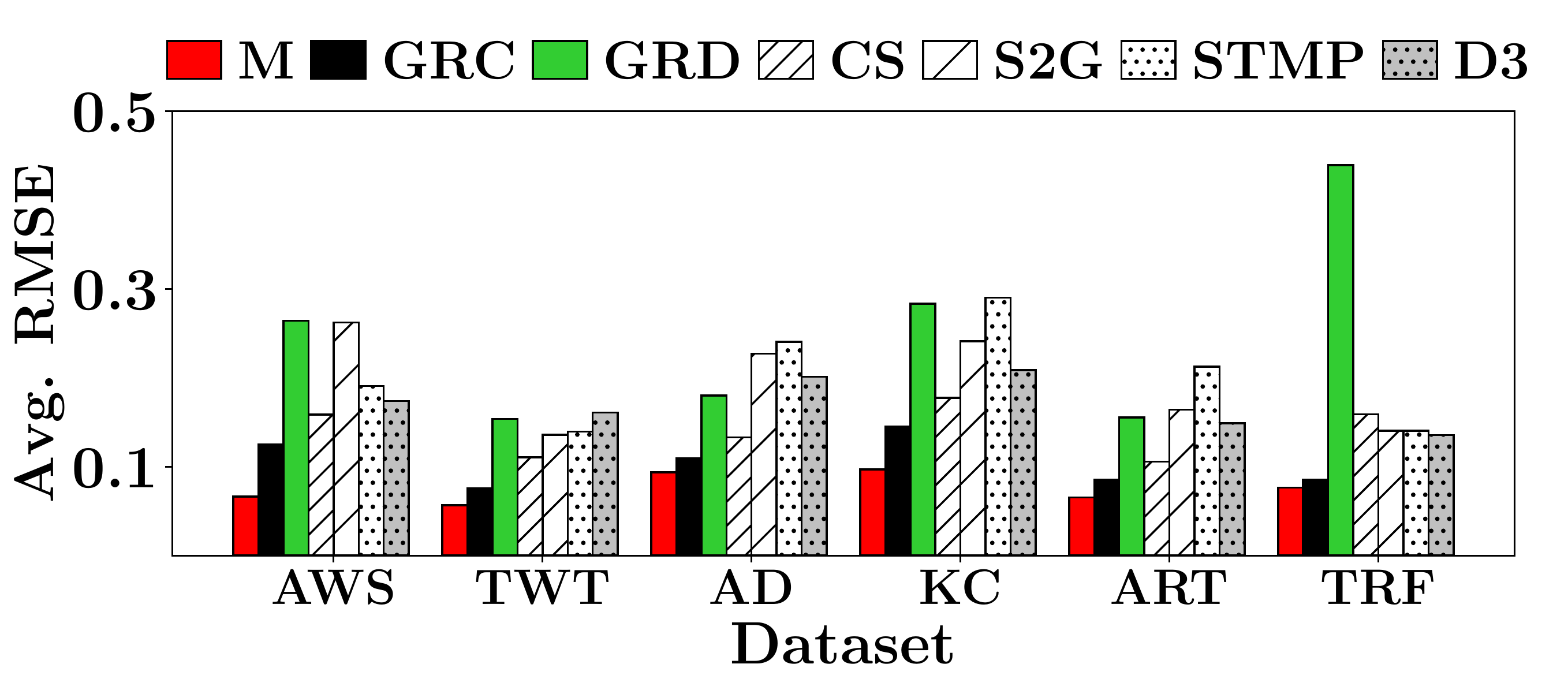}\vspace{-3mm}
    \caption{\mc{The average RMSE, the smaller the better.}}
    \label{fig:ks_statistic}\vspace{-1mm}
\end{figure}
\begin{table}[t]
\begin{tabular}{|l|c|c|c|c|c|c|}
\hline\small
\textbf{Method} & \textbf{AWS}  & \textbf{TWT}  & \textbf{AD}   & \textbf{KC}   & \textbf{ART}  & \textbf{TRF}  \\ \hline
CS     & 0.85 & 0.92 & 0.93 & 0.90 & 0.85 & 0.80 \\ \hline
GRC    & 0.76 & 0.70 & 0.78 & 0.59 & 0.70 & 0.82 \\ \hline
\end{tabular}
\caption{\mc{The reverse factor, the larger the better.}}
\label{tbl:reverse_factor}\vspace{-5mm}
\end{table}

\mc{
We evaluate the effectiveness of an explanation $\mathcal{I}$ using the \emph{root mean square error (RMSE)} between the empirical cumulative functions of $R$ and $T' = T \setminus \mathcal{I}$. The RMSE is defined as
$    \text{RMSE}=\sqrt{\frac{\sum_{x\in R \cup T'}(F_{R}(x) - F_{T'}(x))^2}{|R \cup T'|}}$,
where $F_{R}$ and $F_{T'}$ are the empirical cumulative functions of $R$ and $T'$, respectively. 
A small RMSE value indicates the distributions of $R$ and $T'$ are similar and the explanation $\mathcal{I}$ is good. }

\mc{
We evaluate MOCHE and all baseline methods in RMSE on the failed KS tests of the time series datasets. Figure~\ref{fig:ks_statistic} shows the average RMSE on each data set for each method. }

\mc{
GRC performs best among all baselines, as it generates explanations on failed KS tests by minimizing the largest absolute difference between $F_{R}$ and $F_{T'}$. However, as its non-differential objective function is hard to minimize, it cannot find a good solution to its optimization problem. 
As discussed in Section~\ref{subsec:exp_small}, the explanations generated by the other baselines include many data points that are irrelevant to the failure of the KS tests. Therefore, they tend to have large RMSE.
MOCHE outperforms all baselines. It guarantees to produce the smallest explanations that can reverse the failed KS tests, and thus can guarantee the similarity of the distributions. }

\mc{
Let us examine the explanations on the failed KS test conducted on the COVID-19 dataset.
The two sets are shown as histograms in Figure~\ref{fig:example_covid}a.
Figures~\ref{fig:case_study_aws}a,~\ref{fig:case_study_aws}b, and~\ref{fig:case_study_aws}c
show the histograms of the explanations produced by MOCHE, GRD and D3, respectively. In this case, among all baselines GRD and D3 produce the smallest explanations that can reverse the failed KS test.   The empirical cumulative functions of the reference set, and the test set after removing each explanation are shown in Figure~\ref{fig:case_study_aws}d.}

\begin{figure}[t]
    \centering
    \includegraphics[page=2, width=1\linewidth]{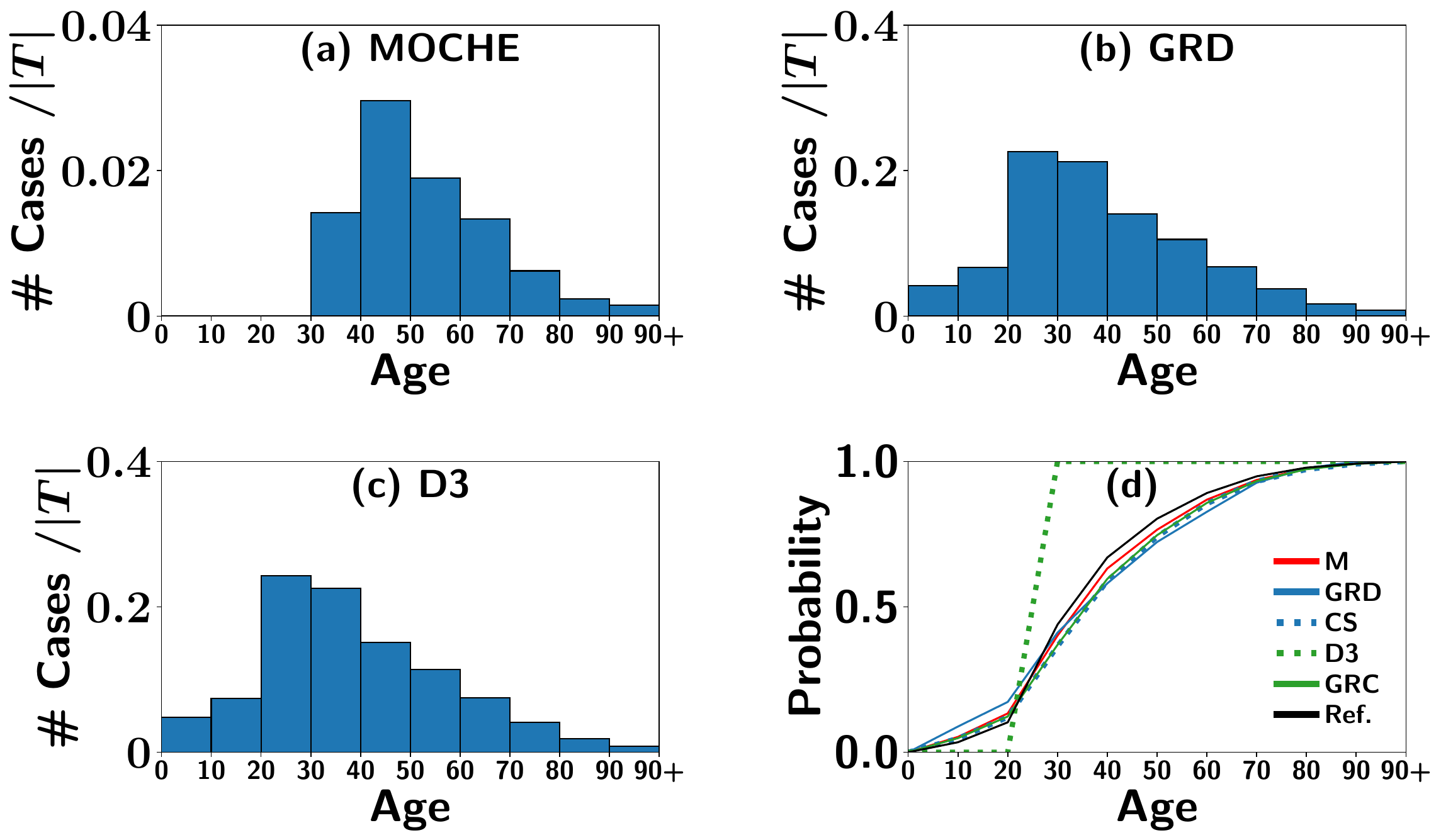}\vspace{-3mm}
    \caption{\mc{The explanations on the failed KS test conducted on the COVID-19 dataset. (d) shows the empirical cumulative functions of the reference set, and the test set after removing the explanations produced by different methods (best viewed in color).}}
    \label{fig:case_study_aws}
\end{figure}

\mc{
Figure~\ref{fig:case_study_aws}a shows that MOCHE selects some data points in the middle and senior age groups. MOCHE mainly selects the data points from age groups that have larger relative frequencies in the test set than in the reference set.
As shown in Figure~\ref{fig:example_covid}b, MOCHE only selects some data points from FHA (Fraser HA), the HA with the largest population. In September, the number of infected middle-aged and senior people in the HA increased dramatically, according to the news reports and analysis in media.
As shown in Figure~\ref{fig:case_study_aws}d, after removing the explanation, the distribution of the test set is most similar to that of the reference set. The results here match the real situation well.
}

\mc{
In terms of explanation size, MOCHE, GRD, and D3 select $291$ $(8.6\%|T|)$, $3,115$ $(92.3\%|T|)$, and $3,370$ $(99.9\%|T|)$ points in their explanations, respectively.
GRD and D3 select almost all data points in the test set. Such explanations are not informative at all.
Please note that STMP and S2G cannot interpret the failed KS test, as they can only work on time series. 
}

\subsection{Efficiency and Scalability}
\label{subsec:exp_scale}

In this subsection, we report the runtime of all methods. In addition, to evaluate the effectiveness of our pruning techniques, we implement a lower-bound ablation MOCHE$^{ns}$ by disabling the pruning using the lower bound of the explanation size (Section~\ref{subsec:binary_search}).

We vary the size of reference sets and test sets. As explained in Section~\ref{sec:settings}, for a given reference/test set size, there are multiple failed KS tests. 
MOCHE constantly outperforms all baseline methods on all datasets. Limited by space, we only report the the average runtime of each method on the largest dataset TWT in Figure~\ref{fig:k_efficiency}a.
The runtime of all methods increases when the test sets become larger. MOCHE is $3$ orders of magnitudes faster than GRC and CS. 

\begin{figure}[t]
    \centering
    \includegraphics[page=1, width=\figwidthone\linewidth]{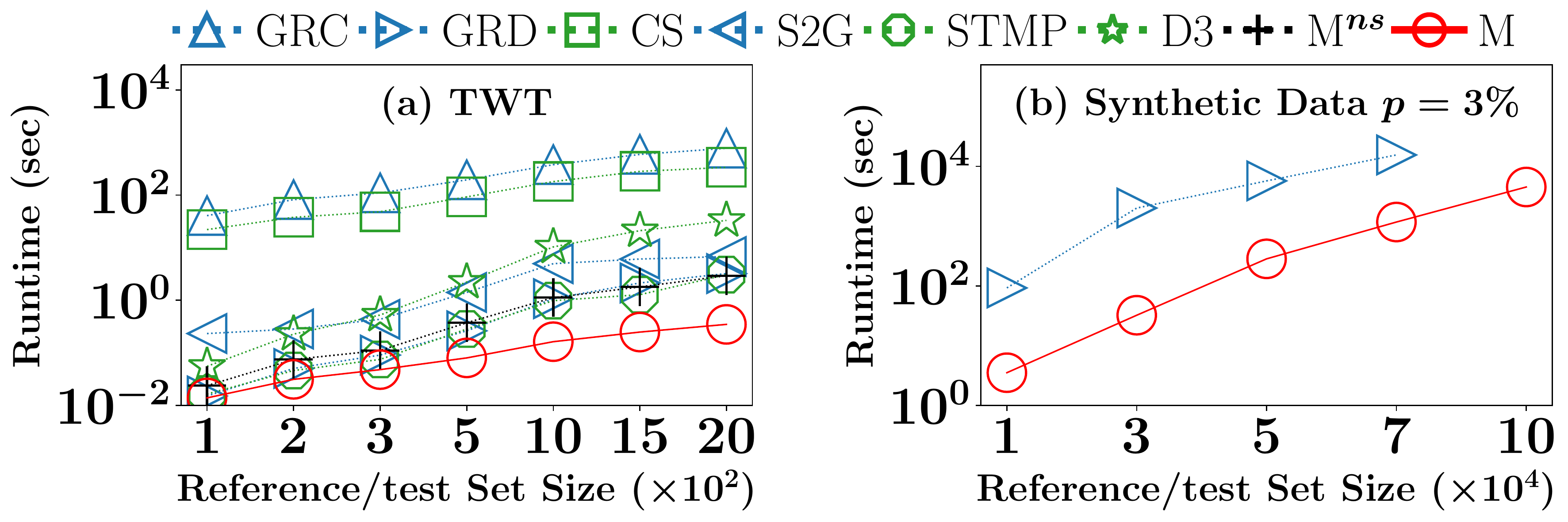}\vspace{-3mm}
    \caption{\mc{The runtime (plotted in logarithmic scale) on data set TWT and the synthetic dataset (best viewed in color).}}
    \label{fig:k_efficiency}
\end{figure}

The poor performance of all baseline methods is due to the cost of conducting huge numbers of KS tests.
GRC needs to conduct $l\cdot m$ KS tests to find an explanation, where $l$ is the number of optimization steps. According to the parameter settings in~\cite{grace}, in the worst case, GRC has to perform $l=10,000$ steps. 
 
CS has to generate a large number of samples to find an explanation, which takes a long time to verify. In the worst case, according to the parameter settings in~\cite{croce2019sparse}, CS has to generate $150,000$ random samples. GRD and D3 need to conduct the KS test after removing each data point.
Since our estimated lower bound on the size $k$ of explanations effectively reduces the search range of $k$, MOCHE interprets failed KS tests faster than MOCHE$^{ns}$.

\begin{figure}[t]
    \centering
    \includegraphics[width=0.65\linewidth, page=1]{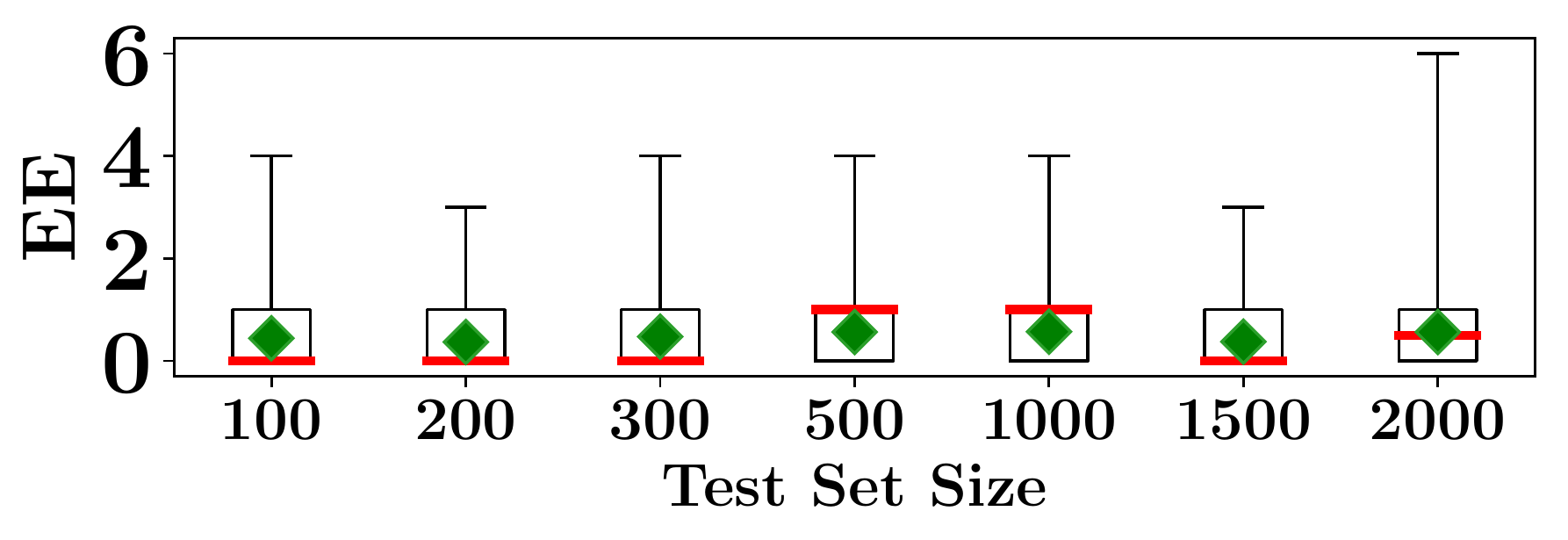}\vspace{-3mm}
    \caption{The estimation errors (EE) of the explanation size.}
\label{fig:error_vs_size}
\end{figure}

\mc{
To comprehensively evaluate the efficiency, we construct large synthetic datasets to further compare the performance of MOCHE and GRD, the most efficient baseline method that can produce comprehensible explanations. Following the idea in~\cite{kifer2004detecting}, we first generate the reference set $R$ and the test set $T$ with the same size $w$ from the normal distribution. Then, we replace a $p$ fraction of $T$ by data points sampled from a uniform distribution between $[-7,  7]$, such that $R$ and $T$ fail the KS test with significance level $\alpha=0.05$. We use a variety of $w$ and $p$ values.
We interpret the failed KS tests with randomly generated preference lists $L$.
Our method constantly outperforms GRD on all these experiments. Limited by space, we only report the  runtime on the synthetic dataset with $p=3\%$ in Figure~\ref{fig:k_efficiency}b.
When $w=10^5$, GRD cannot stop within 2 hours. MOCHE is at least 10 times faster than the most efficient baseline method.
}

To investigate the tightness of the lower bound $\hat{k}$ on the explanation size $k$, we also report the \emph{estimation error (EE)} defined by $k - \hat{k}$.
A small value of EE indicates that our estimated lower bound is tight.
Figure~\ref{fig:error_vs_size} shows the results with respect to different sizes of test sets by box plot~\cite{williamson1989box}. Each bar in the figure shows EE on the KS tests with a specific size of test sets. The upper and lower edges of a box show the first and third quartiles of the estimation errors, respectively. The upper and lower ends of an error bar show the maximum and minimum 
EE, respectively. The red line segment in a box and the green diamond marker show the median and the mean of the estimation errors, respectively. 

For more than $25\%$ of the failed KS tests, our estimated lower bound $\hat{k}$ is equal to the true value of $k$. For more than $75\%$ of the failed KS tests, the estimation errors are up to 1. In the worst case (a KS test with $2,000$ data points in the test set), our estimation error is only $6$, much smaller than the test set size. Besides, we observe that when the test sets become larger, the average value of estimation errors is always smaller than $1$. The results seem to suggest that estimation errors may be treated as a constant in practice. This result is consistent with our observation in Figure~\ref{fig:k_efficiency} that MOCHE is more efficient than MOCHE$^{ns}$.

\section{Conclusions}
\label{sec:conclusion}

In this paper, we tackle the novel problem of producing counterfactual explanations on failed KS tests. We propose the notion of most comprehensible counterfactual explanation, and develop a two-phase algorithm, MOCHE, which guarantees to find the most comprehensible explanation fast. We report extensive experiments demonstrating the superior capability of MOCHE in efficiently interpreting failed KS tests. As future work, we plan to extend MOCHE to interpret failed KS tests conducted on multidimensional data points~\cite{10.1093/mnras/225.1.155, rabanser2019failing}.

\nop{
\begin{acks}
 This work was supported by the [...] Research Fund of [...] (Number [...]). Additional funding was provided by [...] and [...]. We also thank [...] for contributing [...].
\end{acks}
}

\bibliographystyle{ACM-Reference-Format}
\bibliography{sample}


\end{sloppy}
\end{document}